\documentclass[mlmain,onecolumn]{jmlr}

\usepackage{booktabs}
\allowdisplaybreaks

\newcommand{\Zp}{\mathbb{Z}_p}
\newcommand{\Kmax}{K_{\mathrm{max}}}
\newcommand{\Keff}{K_{\mathrm{eff}}}
\newcommand{\Linf}{L_{\infty}}
\newcommand{\hc}{h_{c}}

\jmlrvolume{}
\firstpageno{1}
\jmlryear{}
\jmlrproceedings{}{Preprint. Under review at NeurReps 2026}

\title[Symmetry without a manifold]{Symmetry without a manifold:
intrinsic dimension on orbits}

\author{\Name{Chon-Fai Kam\nametag{\thanks{Corresponding author.}}}
  \Email{dubussygauss@gmail.com}\\
  \addr Dipartimento di Fisica e Chimica, Universit\`a degli Studi di Palermo,
  via Archirafi 36, I-90123 Palermo, Italy\\
  \addr Universit\'e Paris Cit\'e and Universit\'e de La R\'eunion,
  BIGR, INSERM UMR\_S1134, F-75014 Paris, France
  \AND
  \Name{Miloud Bessafi} \Email{miloud.bessafi@gmail.com}\\
  \addr EnergyLab, Universit\'e de La R\'eunion, F-97715 Saint-Denis, France
  \AND
  \Name{Fr\'ed\'eric Cadet} \Email{frederic.cadet.run@gmail.com}\\
  \addr Universit\'e Paris Cit\'e and Universit\'e de La R\'eunion,
  BIGR, INSERM UMR\_S1134, F-75014 Paris, France\\
  \addr PEACCEL, AI for Biologics, F-75013 Paris, France}

\makeatletter\@ifundefined{@org@Ginclude@graphics}{}{\let\Ginclude@graphics\@org@Ginclude@graphics}\makeatother

\begin{document}
\maketitle

\begin{abstract}
The standard geometric derivation of neural scaling exponents takes the
intrinsic dimension of a data manifold as its input. On modular addition in
$\Zp$ that derivation has no input. The exact algebraic solution is an orbit of
$\Zp$ acting by isometries. Transitivity alone makes the ratio statistic
underlying the standard dimension estimator a point mass, so the estimator is
undefined, and here the two nearest neighbour distances coincide exactly. Breaking the symmetry at scale
$\epsilon$ returns a number, but one that tracks $1/\epsilon$ with no scale
free plateau. We show that the failure is general, since on any finite orbit of
a group acting by isometries the estimator reports the resolution at which the
set is probed rather than a dimension. What
replaces the power law is exponential in hidden width,
$L(h)=\Linf+A\exp(-c\,h^{\alpha})$, with $R^{2}$ between $0.982$ and $0.995$
against $0.857$ to $0.906$ for a power law admitting the same floor and fitted
under the same protocol. Where the data supply is sufficient the rate belongs
to the regulariser rather than to the group, since weight decay moves $c$ by a
factor of $47$ while group order moves it by $1.10$, a residual below seed to
seed resolution, for every fixed $\alpha$ between $0.75$ and $2$.
The critical width falls with group order rather than rising, against capacity
counting that assigns a fixed number of neurons to each irreducible
representation.
\end{abstract}

\section{Introduction}

A power law is a strong claim about a system. It says that no scale is
preferred, that the same relative improvement follows from the same relative
investment at every size. Neural scaling laws are reported as power laws
across seven orders of magnitude in model size and across architectures,
modalities and languages \citep{kaplan2020,hoffmann2022,bahri2024}, and the
derivations that explain why the form should be a power law all pass through
an assumption about the data. This paper reports a task on which that
assumption fails in the strongest sense available: the quantity the standard
derivation takes as input does not exist. The failure is structural rather
than numerical, and it does not stay confined to the task, because the object
responsible is a group orbit.

Wherever a representation is a finite orbit of a group acting by isometries,
the two nearest neighbour estimator is degenerate, and any dimension it reports
under perturbation is a property of the probe scale rather than of the set. The
condition is finiteness, not symmetry. A continuous ring attractor or toroidal
population code \citep{gardner2022} is not covered, and a continuous torus
given the same treatment returns a stable value near two. What is covered is
the sampled case, which is the usual experimental one: a population recorded at
$N$ equally spaced stimulus conditions is itself a finite orbit, and an
intrinsic dimension estimated from it reports the probe scale rather than the
geometry of the underlying attractor.

Two derivations dominate. The first routes the exponent through geometry,
obtaining $L\sim N^{-4/d}$ from the cell size a $d$ dimensional manifold
partitioned by $N$ parameters admits \citep{sharma2022}. The second routes it
through counting: if a task decomposes into subtasks whose use frequencies
follow a Zipf distribution and each learned subtask reduces the loss by a fixed
amount, the cumulative loss is again a power law
\citep{michaud2023,brill2024}. The two accounts differ in almost everything
except the ingredient that produces the exponent, which in both cases is a
scale free distribution over structure in the data.

Algebraic tasks supply neither ingredient, and the literature has already
noticed that something goes wrong. These tasks entered the field through
grokking \citep{power2022,liu2023omnigrok,varma2023}, and loss curves on
algorithmic problems show pronounced phase transitions that depart from the
established power law trend \citep{naidu2025}. The geometric relation has been
contradicted directly on a one dimensional regression problem where the
measured exponent was $1$ against a prediction of $4$ \citep{liu2023}. What
has not been asked is what replaces the power law when it fails, and whether
the quantity the geometric derivation requires as input exists at all. For
modular addition it does not, and the reason can be stated before any
measurement: the exact solution is known in closed form
\citep{nanda2023,gromov2023,chughtai2023} and its image is an orbit of $\Zp$
acting by isometries, not a sample from a density.

The governing hypothesis is that the rate of decay factorises as
$\log c=f(\lambda)+g(p)$ with $g$ flat, so that the regularisation budget and
not the order of the group sets the rate, and the design of
Sec.~\ref{sec:hyp} is built so that the two sweeps can disagree. The choice of
architecture is what makes the separation readable, since a two layer network
with quadratic activation admits an exact solution whose Fourier content is
known term by term \citep{gromov2023,doshi2024}.

We contribute the following. Intrinsic dimension is undefined on group orbit
representations, not merely mismeasured: the two nearest neighbour ratio is a
point mass on any orbit of a finite group acting by isometries
(Proposition~\ref{prop:orbit}), equals exactly $1$ for the representations at
issue (Lemma~\ref{lem:coincide}), and under perturbation of scale $\epsilon$
the estimate tracks $1/\epsilon$ with no plateau. The test loss on $\Zp$ is
exponential in hidden width rather than polynomial in parameter count, over
four and a half decades, and a power law admitting the same floor, fitted under
the same protocol, is rejected at every group order and at every fixed exponent
between $0.75$ and $2$. Given sufficient data the rate factorises: weight decay
moves it by a factor of $47$ and group order by $1.10$, a residual below seed
to seed resolution, at every fixed exponent in that range. And the width at which
generalisation appears falls as the group grows, contradicting in sign as well
as magnitude any capacity argument that assigns a fixed number of neurons to
each irreducible representation. On this task the geometric route has no input,
the power law is replaced by an exponential in hidden width, and the rate of
that exponential belongs to the training regime and not to the group.

\section{Problem formulation}
\label{sec:setup}

\subsection{Task and model}

The task is addition in the cyclic group $\Zp$ for $p$ prime. Inputs are pairs
$(a,b)\in\Zp\times\Zp$ encoded as two concatenated one hot vectors, so the
input dimension is $2p$, and the target is $(a+b)\bmod p$ treated as $p$ way
classification. Half of the $p^{2}$ pairs are held out, the partition drawn
once from a fixed seed so that it is identical across every width, weight
decay and initialisation. Table~\ref{tab:notation} collects the notation.

The network follows \citet{gromov2023}. Writing $x$ for the input,
\begin{equation}
f(x) \;=\; W_{2}\,\sigma\!\left(W_{1}x\right),
\qquad \sigma(z)=z^{2},
\label{eq:model}
\end{equation}
with $W_{1}\in\mathbb{R}^{h\times 2p}$ and $W_{2}\in\mathbb{R}^{p\times h}$ and
no biases. The quadratic activation is what puts the Fourier content of the
solution in closed form, since the cross term at
frequency $k$ in the square of a sum of cosines produces
$\cos\omega_{k}(a+b)$ directly. The parameter count is $N=3ph$, so width and
parameter count are proportional at fixed $p$.

Training is full batch AdamW \citep{loshchilov2019}, with weight decay the
principal control variable, swept over
$\{0,\,0.1,\,0.25,\,0.5,\,1,\,2,\,4\}$ and defaulting to $1$ where a single
value is needed. Three seeds are run at every point. Reported quantities are
the cross entropy and the accuracy on the held out half, together with the two
spectral ones of Sec.~\ref{sec:candidates}. Appendix~\ref{app:fit} gives the
optimiser settings, the initialisation and the step budget.

\subsection{The exact solution as a group orbit}
\label{sec:orbit-setup}

The geometry the paper measures is fixed by representation theory before any
network is trained. Writing $\omega_{k}=2\pi k/p$, the real irreducible
representations of $\Zp$ are the trivial one together with $\Kmax=(p-1)/2$ two
dimensional ones, in which $\rho_{k}(m)$ is the rotation by $\omega_{k}m$ and
$k$ and $p-k$ are equivalent. \emph{Frequency} and \emph{irreducible
representation} are therefore the same object counted in two languages, and
$\Kmax$ is the number of either available at group order $p$.

For a nonempty $S\subseteq\{1,\dots,\Kmax\}$ define the embedding
\begin{equation}
\Phi_{S}:\Zp\to\mathbb{R}^{2|S|},
\qquad
\Phi_{S}(n) \;=\; \big(\cos\omega_{k}n,\ \sin\omega_{k}n\big)_{k\in S},
\label{eq:phiS}
\end{equation}
and write $X_{S}=\Phi_{S}(\Zp)$ for its image. The exact algebraic solutions of
the task, known in closed form
\citep{nanda2023,gromov2023,chughtai2023}, place the $p$ input tokens at the
points of $X_{S}$ for some $S$. The group acts on $X_{S}$ by
\begin{equation}
m\cdot\Phi_{S}(n) \;=\; \Phi_{S}(n+m)
\;=\; \Big(\bigoplus_{k\in S}\rho_{k}(m)\Big)\,\Phi_{S}(n),
\label{eq:action}
\end{equation}
which is block diagonal in rotations and hence an isometry of the ambient
space. The action is transitive by construction, and it is free because $p$ is
prime and $S$ is nonempty, so $\Phi_{S}$ is injective and $|X_{S}|=p$.

The object of study is thus a single orbit of a finite group acting by
isometries, not a sample from a density on a manifold, and
Sec.~\ref{sec:manifold} shows that this is what breaks the geometric
derivation.

\subsection{Candidate controlling variables}
\label{sec:candidates}

The literature advances three quantities as what sets performance on a task of
this kind, and the study measures all three across problem sizes rather than at
one. The first is the intrinsic dimension $d$, which the
geometric derivation \citep{sharma2022} takes as its input, estimated by the
two nearest neighbour method of \citet{facco2017} on $X_{S}$ or on the
embeddings a trained network learns. The second is the number of irreducible
representations the network carries, measured by the aggregate participation
ratio $\Keff$ of the spectral power in the embedding block of $W_{1}$ and by
the per neuron ratio obtained by the same construction within a single hidden
unit, both defined in Appendix~\ref{app:coupon}, Eq.~\eqref{eq:keff}. The
third is the width $\hc$ at which held out accuracy first crosses $0.9$, where
a capacity argument would locate the dependence on group order.

\subsection{Hypotheses}
\label{sec:hyp}

Competing accounts of what limits performance on this task make opposite
predictions, and the design is chosen so that they can disagree. Under a \emph{capacity}
account the binding constraint is the number of irreducible representations the
network can carry, so the rate at which the loss falls with width, and the
width $\hc$ at which generalisation appears, should both scale with $\Kmax$ and
hence with $p$. Under a \emph{budget} account the binding constraint is not the
availability of modes but the norm in which they are expressed, so the rate
should be a function of the regularisation strength $\lambda$ alone. Writing
$c$ for that rate, the budget account predicts the factorisation
\begin{equation}
\log c \;=\; f(\lambda) \;+\; g(p),
\qquad g \ \text{constant in}\ p,
\label{eq:fact}
\end{equation}
which is falsified by any resolvable dependence of $g$ on group order.

Sweeping width at fixed data separates capacity from optimisation, and
sweeping weight decay at fixed width separates the norm budget from capacity.
We measure the test loss for eight group orders between $23$ and $113$, across
fourteen widths from $8$ to $128$ and seven weight decay strengths, with three
seeds at every point. Three questions organise what follows: whether $d$
exists on $X_{S}$ at all (Sec.~\ref{sec:manifold}), which functional family
describes $L(h)$ (Sec.~\ref{sec:exp}), and which sweep moves the rate
(Sec.~\ref{sec:fact}).

\begin{figure}[t]
\centering
\includegraphics[width=0.66\textwidth]{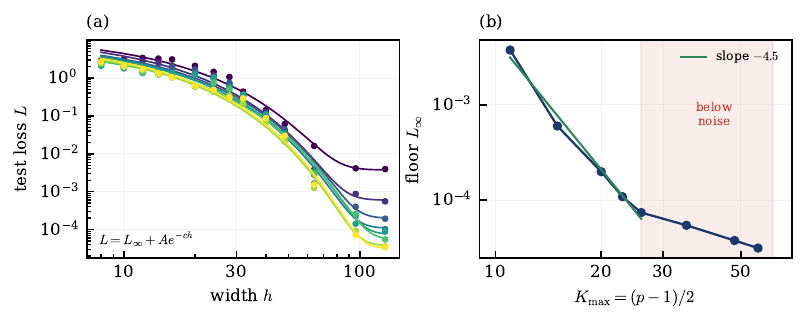}
\caption{Test loss against width at weight decay $1$ with fits to
Eq.~\eqref{eq:law} at $\alpha=1$ (a), and the fitted floor against the number
of available frequencies (b). Beyond $\Kmax=26$ the floor falls into the
plateau noise documented in Appendix~\ref{app:fit} and the flattening there is
not fitted.}
\label{fig:law}
\end{figure}

\section{The loss is exponential in width}
\label{sec:exp}

The first question is whether the received functional form describes the data
at all. Fitting $\log L$ against $\log N$ pooled over all $112$
configurations at weight decay $1$ gives $R^{2}=0.726$, while fitting $\log L$
against $h$ gives $R^{2}=0.924$, and per group order the best power law
reaches $R^{2}$ between $0.857$ and $0.906$. The residuals are not scattered
but arc shaped, giving a Wald--Wolfowitz runs statistic of $z=-2.78$ at every
one of the eight group orders (Appendix~\ref{app:fit}). The local slope varies
in magnitude from $0.04$ to $8.2$ and turns negative near $h=12$ to $16$,
and a power law cannot accommodate a sign change in its own exponent.

A plain exponential improves matters but introduces a new instability, since
the curve saturates at the top of the width range and a form without a floor
compensates by tilting, drifting by up to $105\%$ under choices of cutoff.
Admitting a floor removes it, as Fig.~\ref{fig:law}(a) shows. Writing
\begin{equation}
L(h) \;=\; \Linf \;+\; A\,e^{-c\,h^{\alpha}},
\label{eq:law}
\end{equation}
and fixing $\alpha=1$ for the moment, the fit reaches $R^{2}$ between $0.982$
and $0.995$ across the eight group orders, the rate becomes
$c=0.1180\pm0.0065$, and the drift under truncation falls to $0.3\%$
(Appendix~\ref{app:fit}).

The comparison is not one of parameter count. A power law admitting the same
floor,
\begin{equation}
L(h) \;=\; \Linf \;+\; A\,h^{-\gamma},
\label{eq:powfloor}
\end{equation}
fitted to $\log L$ over the same width range under the same rejection rules and
the same seed aggregation, reaches $R^{2}$ between $0.857$ and $0.906$ against
$0.982$ to $0.995$ for Eq.~\eqref{eq:law}, with $\Delta\mathrm{AIC}$ between
$25.5$ and $40.9$ in favour of the exponential at every group order and at
every fixed $\alpha$ between $0.75$ and $2$. The third parameter is not what
decides it: the floor recovered by Eq.~\eqref{eq:powfloor} is not identified by
the data, since a power law already approaches zero at a polynomial rate and
has nothing for a floor to absorb (Appendix~\ref{app:fit}). The extra
parameter absorbs a systematic feature of the data and thereby stabilises the
parameter of interest.

The exponent $\alpha$ is not determined by these measurements and the paper
does not claim it. Left free it converges to $1.583\pm0.254$, and over
$[0.75,2]$ the fitted $c$ ranges by a factor of $214$ while $R^{2}$ moves only
from $0.978$ to $0.994$ (Appendix~\ref{app:fit}). Section~\ref{sec:fact}
reports what survives the degeneracy, and it survives because the claim made
there concerns how $c$ responds to a change in conditions, not its value.

The floor $\Linf$ falls with group order, from $3.7\times10^{-3}$ at $p=23$ to
$3.1\times10^{-5}$ at $p=113$ as Fig.~\ref{fig:law}(b) shows, its logarithm
correlating with $\log\Kmax$ at $-0.94$. Above $\Kmax=26$ the apparent
flattening is an artefact of measurement rather than a property of the task,
since the loss there executes a random walk within a flat basin whose
amplitude exceeds the differences between adjacent floors
(Appendix~\ref{app:fit}), so the slope is fitted only for $\Kmax\le26$.

\section{The geometric route has no input}
\label{sec:manifold}

Having established that the loss is not a power law, we turn to the derivation
that predicts one. The relation $\alpha_{N}=4/d$ requires a manifold dimension
$d$, and on $\Zp\times\Zp$ the naive answer is $2$. Fitted over the width range
of Sec.~\ref{sec:exp}, the best power law returns an exponent of $3.87$ at
$p=47$ and between $2.87$ and $4.20$ across the eight group orders, implying
dimensions between $0.95$ and $1.39$. The implied dimension is not close to
$2$, and it is not stable, since the exponent rises by half again from $p=23$ to
$p=113$, where the derivation supplies no mechanism for it to move at all. What
the geometric route delivers here is not a number to be checked against $4/d$.

The deeper problem is that $d$ does not exist here. The standard estimator of
\citet{facco2017} recovers $d$ as the shape of the Pareto law followed by the
ratio $\mu_{i}=r_{2}/r_{1}$ of each point's two nearest neighbour distances,
and is accurate on manifolds of known dimension even at the sample sizes
available here (Appendix~\ref{app:orbit}).

On the exact algebraic solution it returns nothing at all. The reason is
structural and is stated as follows.

\begin{proposition}[Degeneracy on a group orbit]
\label{prop:orbit}
Let $G$ be a finite group acting transitively on a finite set
$X\subset\mathbb{R}^{m}$ by isometries. Then $\mu_{i}=r_{2}(i)/r_{1}(i)$ takes
the same value for every $i\in X$, the empirical distribution of $\mu$ is a
point mass, and the two nearest neighbour regression is undefined.
\end{proposition}

Transitivity alone suffices, and the proof is one line: an isometry taking
$x$ to $y$ carries the distances out of $x$ onto those out of $y$, so
$r_{1}$, $r_{2}$ and their ratio agree at every point and the regression has
no variation in its design variable. The action of Eq.~\eqref{eq:action} is in
addition free, which fixes $|X_{S}|=p$.
Appendix~\ref{app:orbit} proves the proposition and shows that at full
frequency content the $p$ embedding vectors form a regular simplex with
pairwise cosine similarity $-1/(p-1)$ identically, which at $p=47$ is
$-0.0217$ with a sample standard deviation of zero to machine precision.
Breaking the symmetry with additive noise of scale $\epsilon$ produces a
number, but not a stable one, tracking $1/\epsilon$ over a factor of thirty
in the probe scale with no plateau, where a genuine torus given the same
treatment returns a stable value near $2$ (Appendix~\ref{app:orbit}). The
dimension of an algebraic solution is not a property of the set but of the
resolution at which it is probed, and $4/d$ has no input to take. What is not
immediate is the scope. This is not small sample noise, since it holds at
every $N$ and every embedding; not special to $\Zp$, since transitivity is the
only hypothesis; and not repaired by perturbation, since the number that then
appears tracks the probe rather than the set.

Trained networks reproduce this: at $p=47$ and width $64$, at unit test
accuracy, the estimator returns $68$ on the learned embeddings against an
ambient dimension of $64$, and grows with width. The simplex describes full
frequency content, not what a trained network shows, and
\citet{beyondnc2026} report a low rank cyclic geometry instead; both are
orbits, so Proposition~\ref{prop:orbit} covers either
(Appendix~\ref{app:orbit}).

\begin{figure}[t]
\centering
\includegraphics[width=0.66\textwidth]{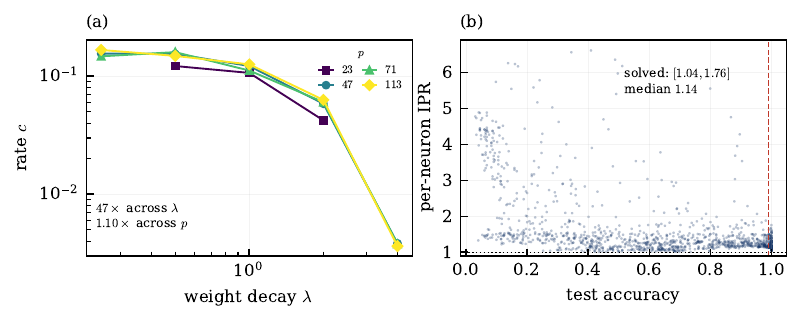}
\caption{The rate $c$ against weight decay for four group orders (a). The
curves for $p\ge47$ coincide to within measurement error while the rate itself
moves by a factor of $47$. Per neuron spectral participation ratio against
held out accuracy over the full grid (b), showing that rank one structure
holds only where the task is solved.}
\label{fig:fact}
\end{figure}

\section{The irreducible representation count does not control the loss}
\label{sec:irreps}

If the geometric variable is unavailable, the mechanistic one seems obvious.
Each neuron converges to a single irreducible representation under gradient
flow \citep{he2026spectral,he2026modular}, which makes the count of distinct
represented modes well defined, and it is the quantity we expected to organise
the data. It does not. Within a single group order the correlation is
convincing, returning $R^{2}=0.90$ at $p=47$. Across group orders it
collapses, the loss ranging over a factor of $2500$ at $\Keff\approx20$ and
$\Keff$ returning $R^{2}=0.590$ pooled against $R^{2}=0.917$ for width alone.
The within group correlation is an artefact of an occupancy process rather
than a measure of capacity: a zero parameter occupancy model accounts for
$R^{2}=0.907$ of the variance in $\Keff$. It is visible only across problem
sizes, since within a single size capacity and sampling are confounded by
construction. The claim delimits \citet{he2026modular} without disagreeing
with them.
Appendix~\ref{app:irreps} gives the measurements and a second finding, that
regularisation imposes the rank one structure of the neurons rather than the
architecture supplying it (Fig.~\ref{fig:fact}(b)).

\section{The rate factorises}
\label{sec:fact}

What remains is the rate. Weight decay is the natural candidate, since
Appendix~\ref{app:exact} shows the unconstrained frequency restricted solution
already achieves zero loss with four frequencies, so the constraint that binds
cannot be the availability of modes. Sweeping weight decay at fixed group order
moves the rate by a factor of $47$, from $0.171$ at weight decay $0.1$ to
$0.00364$ at weight decay $4$, monotone decreasing in the mean over group
orders once runs that fail to solve
the task are excluded. Retaining them produces a spurious maximum at
intermediate weight decay, because the three parameter form fitted to an
unsolved run places the recovered floor above the bulk of the data.
Appendix~\ref{app:fit} states the validity criterion, which rejects a fit when
the recovered floor exceeds three times the smallest observed loss.

Against that factor of $47$ the group order does almost nothing. This is the
test of the budget hypothesis of Sec.~\ref{sec:hyp}, Eq.~\eqref{eq:fact}. Over
four group orders and five weight decay strengths, tabulated in
Table~\ref{tab:fact}, the decomposition returns a spread in $g$ of a factor of
$1.097$ among $p\ge47$, against the factor of $47$ contributed by $f$, and
Fig.~\ref{fig:fact}(a) shows the four curves. The residual spread is not merely
small but unresolvable: in the regime where the floor is identified the
standard error of a three seed mean is $4.3\%$ and the observed spread across
group orders is also $4.3\%$, so group order dependence, if present, is below
what this experiment can detect (Appendix~\ref{app:fit}). Because the exponent
$\alpha$ is not identified, the factorisation must also be checked against the
choice of fixing it at unity. Repeating the fit at $\alpha\in\{0.75,\dots,2\}$
moves the rate by a factor of $214$ while the coefficient of variation across
group order stays between $4.0\%$ and $6.8\%$ (Table~\ref{tab:alpha}), so the
claim concerns the structure of the dependence, not a number.

The single exception is instructive. At $p=23$ the rate sits a factor of
$0.793$ below the others, a real deviation of $4.8$ standard errors, and that
group order also fails outright at low regularisation. Its training set
contains $264$ pairs against $6384$ at $p=113$. The factorisation is a
statement about the regime in which data are sufficient, and $p=23$ marks one
edge of it. There is a second edge in the same direction: raising the training
fraction from a half to four fifths moves the rate at both activations, and
under the quadratic one it carries the residual past zero
(Appendix~\ref{app:frac}).

Whether the factorisation belongs to the task or to the quadratic activation
is settled in part by repeating the sweep with ReLU. The functional form
transfers, the exponential family reaching $R^{2}=0.963$ against $0.869$ for
the best power law. The factorisation does not, since the rate differs between group
orders by $61\%$ on average, and truncation does not account for it. The
critical width tracks the discrepancy, in the pattern of Sec.~\ref{sec:hc},
where the quadratic activation itself departs at $p=23$ with $264$ training
pairs. The factorisation holds wherever data are plentiful relative to what the
architecture needs per mode, and ReLU reaches that boundary at a larger group
order. Supplying more data closes the gap (Appendix~\ref{app:frac}). Appendix~\ref{app:relu} gives the protocol and the
full comparison.

\begin{figure}[t]
\centering
\includegraphics[width=0.66\textwidth]{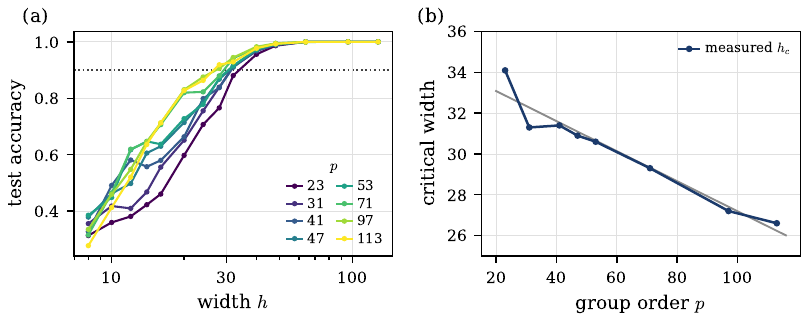}
\caption{Held out accuracy against width at weight decay $1$ for eight group
orders (a), and the resulting critical width against group order (b), with the
least squares line. The requirement falls as the group grows, where capacity
counting that assigns a fixed number of neurons per irreducible representation
requires it to rise.}
\label{fig:hc}
\end{figure}

\section{Critical width falls as the group grows}
\label{sec:hc}

The factorisation says that group order does not enter the rate. Whether it
enters anywhere is best asked of $\hc$, where a capacity argument would put it.
Defining $\hc$ as the width at which held out accuracy crosses $0.9$ by linear
interpolation, it falls from $34.1$ at $p=23$ to $26.6$ at $p=113$, a linear
fit against $p$ giving a slope of $-0.0738$ with $R^{2}=0.934$ as
Fig.~\ref{fig:hc}(b) shows. The requirement falls as the group grows, by
twenty two percent over a range of $4.9$ in the order of the group.

This is the opposite of what capacity counting predicts. Any account assigning
a fixed number of neurons to each irreducible representation makes $\hc$ grow
with $p$, and our own Proposition~\ref{prop:construct} is of that form, needing
$4(p-1)$ neurons at full frequency content, whereas the expressible class of a
monomial activation is fixed by its degree alone and does not grow with $p$
\citep{kam2026algebraic}. Such bounds give a width at which
the solution can be represented and not the width at which training finds it,
but the disagreement here is in sign as well as magnitude and cannot be
absorbed by a constant.

The training set explains the direction, not the representation theory. Pairs
grow as $p^{2}/2$ while frequencies grow as $p/2$, so larger groups are better
supplied with data per mode and overfit less, and \citet{tian2025} prove that $O(M\log M)$ samples
suffice on a group arithmetic task of order $M$, which grows more slowly than
the pairs the task supplies. This also accounts for the failure of $p=23$ at
low regularisation reported in Sec.~\ref{sec:fact}. An independent argument
reaches the same conclusion from the representation side, since
\citet{beyondnc2026} derive a critical regularisation strength scaling as the
inverse of the number of classes: larger groups should tolerate weaker weight
decay, which is what we observe.

\section{Discussion}

The power law form of neural scaling laws is a property of the data, not of
learning as such, and modular addition supplies neither ingredient the
derivations require. Its exact solution is a group orbit and not a sampled
manifold, its irreducible representations are equivalent under the symmetry of
the group and carry no preferred scale, and they combine coherently rather than
additively, since the logit is built by constructive interference of modes into
a peak. Exponential scaling is not
without precedent \citep{sorscher2022,liu2025superposition}, though there the
mechanism is a change in the effective distribution over structure rather than
a property of the task. Whether the form follows from the coherence is left
open: the stretched exponential of Appendix~\ref{app:exact} depends on group
order where the measured rate does not.

Both departures from the factorisation, the quadratic activation at $p=23$ and
ReLU at $p=47$, occur where the training set is smallest relative to the width
the architecture needs, and in both the fitted rate falls below the common
value. That predicts more data closes the deficit, and it does, by an amount
that tracks the size of the deficit at both activations
(Appendix~\ref{app:frac}). Data supply bounds the factorisation, not the
activation.

The most portable result is methodological. Three quantities that behave like
controlling variables within a single problem size turn out not to be
controlling variables at all: the intrinsic dimension, which is undefined on
group orbits yet returns a plausible number under perturbation; the aggregate
spectral occupancy, which tracks an occupancy process rather than capacity;
and the fitted power law exponent itself, which is stable enough within one
group order to be reported but moves by half again across them, where the
derivation that produces it gives no reason for movement. Each would have
survived a study at a single problem size. Wherever a representation carries a
group action, the estimators that summarise its geometry must be validated
against that action before their output is read as a dimension. The validation
is two checks, both available beforehand: the empirical distribution of the
ratio statistic, and the stability of the estimate under a change of probe
scale. The torus control of Appendix~\ref{app:orbit} passes both and the orbit
fails both.

The non abelian case is the natural continuation, with $\Kmax$ replaced by
$\sum_{\rho}d_{\rho}$: whether the factorisation survives that replacement
would say whether the dimensions of the representations enter the rate
\citep{chughtai2023,stander2024}.

\bibliography{refs_neurreps}

\appendix

\section{The frequency restricted solution and its loss}
\label{app:exact}

\begin{table}[h]
\caption{Notation.}
\label{tab:notation}
\centering
\begin{tabular}{ll}
\toprule
Symbol & Meaning \\
\midrule
$p$ & order of the cyclic group, prime \\
$\omega_{k}$ & $2\pi k/p$ \\
$\rho_{k}$ & real two dimensional irreducible representation at frequency $k$ \\
$\Kmax$ & $(p-1)/2$, number of available frequencies, equivalently of irreps \\
$S$ & subset of frequencies carried by a solution \\
$\Phi_{S}$, $X_{S}$ & embedding of Eq.~\eqref{eq:phiS} and its image \\
$h$ & hidden width; $N=3ph$ the parameter count \\
$\lambda$ & weight decay strength \\
$L(h)$, $\Linf$ & test cross entropy and its fitted floor \\
$A$, $c$, $\alpha$ & amplitude, rate and exponent of Eq.~\eqref{eq:law} \\
$\gamma$ & exponent of the matched power law, Eq.~\eqref{eq:powfloor} \\
$\beta_{k}$, $B$ & mode amplitudes and their $\ell_{2}$ budget \\
$\Keff$ & aggregate spectral participation ratio, Eq.~\eqref{eq:keff} \\
$\hc$ & width at which held out accuracy crosses $0.9$ \\
$d$ & intrinsic dimension as estimated by \citet{facco2017} \\
\bottomrule
\end{tabular}
\end{table}

This appendix works out what the exact algebraic solution costs when it is
restricted to a subset of the available irreducible representations and its
logits are held to a fixed norm. Four things are established. The architecture
of Eq.~\eqref{eq:model} represents the target logit exactly, and the explicit
construction that does so is far more expensive in width than what the trained
network finds. The resulting logit is a Dirichlet kernel whose behaviour at
full frequency content is exactly a discrete delta. Without a norm constraint
the loss reaches zero at very small mode counts, so mode availability is not
what limits the trained network. Under a norm constraint the loss becomes a
stretched exponential in the mode count, but with a range of validity narrower
than a naive expansion suggests and with an optimal amplitude assignment that
is not the uniform one.

\subsection{From the architecture to the logit}

Write $W_{1}=[U;V]$ with $U,V\in\mathbb{R}^{p\times h}$, so that the
preactivation of hidden unit $j$ on the input pair $(a,b)$ is $U_{aj}+V_{bj}$
and the logit assigned to class $c$ is
\begin{equation}
z_{c}(a,b) \;=\; \sum_{j=1}^{h}\big(U_{aj}+V_{bj}\big)^{2}\,W_{2,jc}.
\label{eq:logitraw}
\end{equation}
The quadratic activation in Eq.~\eqref{eq:logitraw} matters because the
square of a sum of two cosines contains a cross term at the sum frequency.
Writing $\varphi_{a}=\omega_{k}a$ and $\varphi_{b}=\omega_{k}b$, expanding
$(\cos\varphi_{a}+\cos\varphi_{b})^{2}$ produces
$2\cos\varphi_{a}\cos\varphi_{b}=\cos(\varphi_{a}-\varphi_{b})
+\cos(\varphi_{a}+\varphi_{b})$, so the
dependence on $a+b$ that the task requires is already present after one
nonlinearity. The difficulty is that the same expansion also produces
$\cos(\varphi_{a}-\varphi_{b})$ and the two squared terms, none of which
depend on $a+b$, and these must be removed.

They can be removed exactly. The following construction uses eight hidden
units per frequency and second layer weights taking values in
$\{0,\pm\tfrac14\}$.

\begin{proposition}[Exact representability]
\label{prop:construct}
Fix $S\subseteq\{1,\dots,\Kmax\}$ and amplitudes $\beta_{k}$. There is a choice
of $W_{1}$ and $W_{2}$ with $h=8|S|$ for which
\begin{equation}
z_{c}(a,b) \;=\; \sum_{k\in S}\beta_{k}\cos\!\big(\omega_{k}(a+b-c)\big)
\label{eq:target}
\end{equation}
holds identically on $\Zp\times\Zp$.
\end{proposition}

\begin{proof}
Fix $k\in S$ and write $\varphi_{a}=\omega_{k}a$ and $\varphi_{b}=\omega_{k}b$
as above. Allocate four units carrying the preactivations
$\cos\varphi_{a}\pm\cos\varphi_{b}$ and $\sin\varphi_{a}\pm\sin\varphi_{b}$.
Squaring and combining with signs $+,-,-,+$ gives
\begin{align}
(\cos\varphi_{a}+\cos\varphi_{b})^{2}-(\cos\varphi_{a}-\cos\varphi_{b})^{2}
&= 4\cos\varphi_{a}\cos\varphi_{b},\\
(\sin\varphi_{a}+\sin\varphi_{b})^{2}-(\sin\varphi_{a}-\sin\varphi_{b})^{2}
&= 4\sin\varphi_{a}\sin\varphi_{b},
\end{align}
and subtracting the second line from the first leaves
$4\cos(\varphi_{a}+\varphi_{b})$. Every term that does not depend on
$a+b$ has cancelled. Allocate four further units carrying
$\cos\varphi_{a}\pm\sin\varphi_{b}$ and $\sin\varphi_{a}\pm\cos\varphi_{b}$,
whose squared differences give $4\cos\varphi_{a}\sin\varphi_{b}$ and
$4\sin\varphi_{a}\cos\varphi_{b}$, and whose sum is
$4\sin(\varphi_{a}+\varphi_{b})$. Assigning
second layer weights $\tfrac{1}{4}\beta_{k}\cos\omega_{k}c$ to the first group
and $\tfrac{1}{4}\beta_{k}\sin\omega_{k}c$ to the second and summing over $k$
produces
\begin{equation}
\sum_{k\in S}\beta_{k}\Big[\cos\big(\omega_{k}(a+b)\big)\cos\omega_{k}c
+\sin\big(\omega_{k}(a+b)\big)\sin\omega_{k}c\Big],
\end{equation}
which is Eq.~\eqref{eq:target} by the cosine subtraction formula.
\end{proof}

We have verified Eq.~\eqref{eq:target} numerically against this construction
at $p=47$ for $|S|=1$, $4$ and $23$, with agreement to $10^{-14}$.

The width Proposition~\ref{prop:construct} consumes is worth comparing
against measurement.
At $p=47$ with all frequencies retained it needs $8\Kmax=184$ hidden units,
whereas the critical width reported in Sec.~\ref{sec:hc} is $30.9$. The trained
network is therefore about six times more economical than the most obvious
exact solution, which is consistent with the finding of
Refs.~\citep{he2026spectral,he2026modular} that a single unit carries a single
irreducible representation and that the unwanted terms are suppressed by
incoherent cancellation across units rather than removed exactly. The
construction above should be read as a proof of representability rather than
as a model of what gradient descent finds. That the solutions selected by
training are the max margin ones, and that these use Fourier features on
algebraic tasks, is established by \citet{morwani2024}.

Which targets a quadratic network can reach at all is a prior question,
separate from the width reaching one of them consumes.
\citet{kam2026algebraic} settle it for the monomial family $\sigma(z)=z^{k}$
on roots of unity inputs, where the reachable functions span the $k+1$
characters whose frequency pair sums to $k$, and a target outside that span
cannot be fitted even on the training set. Proposition~\ref{prop:construct}
sits on the far side of that question. The target of Eq.~\eqref{eq:target} is
reachable, and what the width buys is the removal of the terms the square
produces alongside it.

\subsection{The logit as a Dirichlet kernel}

Everything downstream depends on the shape of Eq.~\eqref{eq:target} as a
function of the class index, so we record it in closed form. Taking
$S=\{1,\dots,K\}$ with unit amplitudes and writing $s=a+b-c \bmod p$, the
logit is the partial sum
\begin{equation}
D_{K}(s) \;=\; \sum_{k=1}^{K}\cos(\omega_{k}s).
\label{eq:dirichlet0}
\end{equation}

\begin{lemma}[Closed form]
\label{lem:dirichlet}
For $s\not\equiv0 \bmod p$,
\begin{equation}
D_{K}(s) \;=\; -\frac{1}{2}
\;+\;\frac{\sin\!\big((2K+1)\pi s/p\big)}{2\sin(\pi s/p)},
\label{eq:dirichlet}
\end{equation}
while $D_{K}(0)=K$.
\end{lemma}

\begin{proof}
Write $\theta=2\pi s/p$, which is not a multiple of $2\pi$, so
$\sin(\theta/2)\ne0$. The product to sum identity gives
$2\sin(\theta/2)\cos(k\theta)=\sin\big((k+\tfrac12)\theta\big)
-\sin\big((k-\tfrac12)\theta\big)$. Summing over $k$ from $1$ to $K$ makes the
right hand side telescope to
$\sin\big((K+\tfrac12)\theta\big)-\sin(\theta/2)$. Dividing by
$2\sin(\theta/2)$ and substituting $\theta=2\pi s/p$ gives
Eq.~\eqref{eq:dirichlet}. The value at $s=0$ is immediate from
Eq.~\eqref{eq:dirichlet0}.
\end{proof}

The case of full frequency content is special and exact.

\begin{corollary}[Delta kernel]
\label{cor:delta}
Let $p$ be odd and $\Kmax=(p-1)/2$. Then $D_{\Kmax}(0)=\Kmax$, and
$D_{\Kmax}(s)=-\tfrac12$ for every $s\not\equiv0$.
\end{corollary}

\begin{proof}
At $K=\Kmax$ the numerator of Eq.~\eqref{eq:dirichlet} is
$\sin\big((2\Kmax+1)\pi s/p\big)=\sin(\pi s)$, which vanishes for every
integer $s$.
\end{proof}

The kernel at full frequency content is thus a discrete delta sitting on a
constant background, and the separation between the correct class and every
competitor is exactly $\Kmax+\tfrac12$ per unit amplitude. Numerical
evaluation at $p=47$ agrees to $2\times10^{-14}$. This is the algebraic
statement behind the simplex geometry established in
Appendix~\ref{app:orbit}, since a kernel that is constant off the diagonal is
precisely what makes all pairwise inner products equal.

\subsection{Cross entropy}

Because the group acts transitively on the input pairs, the loss does not
depend on which pair is presented, and the cross entropy for the solution
\eqref{eq:target} with uniform amplitude $\beta$ is
\begin{equation}
L(K,\beta) \;=\; \log\!\Big(1+\sum_{s\ne0}
e^{-\beta\left(K-D_{K}(s)\right)}\Big).
\label{eq:ce}
\end{equation}
Two regimes follow according to whether $\beta$ is free.

\subsubsection{Unconstrained amplitude}

Suppose $\beta$ may grow without bound. Every exponent in
Eq.~\eqref{eq:ce} is then driven to $-\infty$ provided $K-D_{K}(s)>0$ for all
$s\ne0$, that is provided the kernel has a strict maximum at the origin. By
Lemma~\ref{lem:dirichlet} the second term of $D_{K}(s)$ is bounded by
$1/(2\sin(\pi/p))$ in absolute value, so the condition holds for every $K\ge1$
once $p$ is large enough, and it holds for every $K\ge1$ at the values of $p$
studied here. Hence $L\to0$ for any nonempty frequency set.

The rate at which this happens is fast. Optimising $\beta$ at each $K$ for
$p=47$ drives the loss below $10^{-12}$ from $K=4$ onwards, whereas trained
networks whose measured mode count is comparable sit near $L\approx3$. The
gap between these two numbers is the reason Sec.~\ref{sec:fact} looks to the
norm budget rather than to mode availability for the binding constraint.

\subsubsection{Constrained amplitude}

Weight decay does not permit unbounded logits, so the relevant question is how
well a fixed budget can be spent. Allow the amplitudes to differ and write
$\boldsymbol{\beta}=(\beta_{k})_{k\in S}$ subject to $\|\boldsymbol{\beta}\|_{2}=B$. The
logit separation between the correct class and the competitor at offset $s$ is
$\sum_{k\in S}\beta_{k}\big(1-\cos(\omega_{k}s)\big)$, so the quantity that
controls the loss is the worst case separation
\begin{equation}
\Delta_{S}(\boldsymbol{\beta}) \;=\; \min_{s\ne0}\
\sum_{k\in S}\beta_{k}\big(1-\cos(\omega_{k}s)\big),
\label{eq:gapdef}
\end{equation}
and Eq.~\eqref{eq:ce} is bounded by
$L\le\log\big(1+(p-1)e^{-\Delta_{S}(\boldsymbol{\beta})}\big)$ with equality when a
single competitor dominates. Choosing the best allocation of a fixed budget is
therefore the max-min problem
$\Delta^{*}_{S}(B)=\max_{\|\boldsymbol{\beta}\|=B}\Delta_{S}(\boldsymbol{\beta})$.

At full frequency content the answer is uniform allocation, and
Corollary~\ref{cor:delta} makes it exact.

\begin{proposition}[Optimal allocation]
\label{prop:uniform}
Let $S=\{1,\dots,\Kmax\}$. The uniform allocation
$\beta_{k}=B/\sqrt{\Kmax}$ is optimal, attaining
\begin{equation}
\Delta^{*}_{S}(B) \;=\; \frac{B\,(\Kmax+\tfrac12)}{\sqrt{\Kmax}}
\;=\; \frac{B\,p}{\sqrt{2(p-1)}},
\label{eq:deltastar}
\end{equation}
and the loss satisfies
\begin{equation}
L \;=\; \log\!\Big(1+(p-1)\,
\exp\!\Big[-B\,\frac{\Kmax+\tfrac12}{\sqrt{\Kmax}}\Big]\Big)
\;\simeq\; (p-1)\,e^{-B\sqrt{\Kmax}}
\label{eq:fullK}
\end{equation}
for large $B$.
\end{proposition}

\begin{proof}
Write $M_{sk}=1-\cos(\omega_{k}s)$ for $s\ne0$ and $k\in S$, so that
$\Delta_{S}(\boldsymbol{\beta})=\min_{s\ne0}(M\boldsymbol{\beta})_{s}$.

Uniform allocation attains the stated value. Setting $\beta_{k}=\beta$ for all
$k$ gives $(M\boldsymbol{\beta})_{s}=\beta\big(\Kmax-D_{\Kmax}(s)\big)$, which by
Corollary~\ref{cor:delta} equals $\beta(\Kmax+\tfrac12)$ for every $s\ne0$.
The minimum is therefore attained simultaneously at all $p-1$ competitors, and
substituting $\beta=B/\sqrt{\Kmax}$ with $\Kmax+\tfrac12=p/2$ gives
Eq.~\eqref{eq:deltastar}.

No allocation does better. Since $\sum_{s=0}^{p-1}\cos(\omega_{k}s)=0$ for
$k\not\equiv0$, we have $\sum_{s\ne0}\cos(\omega_{k}s)=-1$ and hence
$\sum_{s\ne0}M_{sk}=(p-1)+1=p$ for every $k$. Summing the components of
$M\boldsymbol{\beta}$ therefore gives $p\sum_{k}\beta_{k}$, so the minimum component
obeys
\begin{equation}
\Delta_{S}(\boldsymbol{\beta}) \;\le\; \frac{p}{p-1}\sum_{k\in S}\beta_{k}
\;\le\; \frac{p}{p-1}\sqrt{\Kmax}\,\|\boldsymbol{\beta}\|_{2}
\;=\; \frac{B\,p}{\sqrt{2(p-1)}},
\end{equation}
the second step by Cauchy--Schwarz and the last by
$\Kmax=(p-1)/2$. The bound coincides with the value attained above, and
Cauchy--Schwarz is tight only for uniform $\boldsymbol{\beta}$.

Substituting $\Delta^{*}$ into Eq.~\eqref{eq:ce}, with all $p-1$ competitors
contributing equally, gives Eq.~\eqref{eq:fullK}.
\end{proof}

Since $\Kmax=(p-1)/2$, Eq.~\eqref{eq:fullK} says that at full frequency
content the loss falls as $\exp(-B\sqrt{p/2})$ up to the prefactor. Holding a
small target loss $L$ fixed as the group grows therefore requires
$B\simeq\sqrt{2}\,\log\!\big((p-1)/L\big)/\sqrt{p}$, which decreases with
$p$. The measured Frobenius norm of $W_{1}$ does the opposite, growing from
about $46$ at $p=47$ to about $77$ at $p=113$. The two are not the same
quantity, since $B$ is the norm of the coefficient vector in the logit
expansion while $\|W_{1}\|_{F}$ also carries the width and the parametrisation
of the construction, but the comparison gives no support to the budget account
and we do not claim it does.

For proper subsets the uniform allocation is not optimal, and the discrepancy
is not small. Solving Eq.~\eqref{eq:gapdef} numerically at $p=47$ by
Nelder--Mead from twelve random starts gives worst case separations exceeding
the uniform value by $25\%$ at $K=4$ and by $28\%$ at $K=10$, while at
$K=\Kmax$ the optimiser recovers the uniform allocation to within $1.6\%$, as
Proposition~\ref{prop:uniform} requires. The reason is that the constraint set
loses its symmetry once frequencies are removed, so the coefficients
$1-\cos(\omega_{k}s)$ no longer take a common value and the budget is better
spent on the frequencies that separate the nearest competitor.

With that caveat recorded, the uniform allocation remains the natural
reference because it is what a network with no preference among its retained
modes would produce. Under it, $\beta=B/\sqrt{K}$ and the separation is
governed by the dominant competitor. Numerically the dominant competitor is
$s=1$ for every $K$ below $\Kmax$ at $p=47$, and expanding
$\cos(2\pi k/p)=1-2\pi^{2}k^{2}/p^{2}+O(k^{4}/p^{4})$ gives
\begin{equation}
K-D_{K}(1) \;=\; \frac{2\pi^{2}}{p^{2}}\sum_{k=1}^{K}k^{2}
+O\!\left(\frac{K^{5}}{p^{4}}\right)
\;=\; \frac{2\pi^{2}}{3}\,\frac{K^{3}}{p^{2}}
+O\!\left(\frac{K^{2}}{p^{2}}\right)
+O\!\left(\frac{K^{5}}{p^{4}}\right),
\end{equation}
so that
\begin{equation}
\Delta_{K}(B) \;\simeq\; \frac{2\pi^{2}}{3}\,\frac{B\,K^{5/2}}{p^{2}} .
\label{eq:gap}
\end{equation}
The loss is then a stretched exponential in the mode count with exponent
$5/2$ and an explicit $p^{-2}$ prefactor.

The range over which Eq.~\eqref{eq:gap} is usable is narrower than the
derivation suggests, and it is worth being explicit about where it fails.
Table~\ref{tab:kexp} compares the expansion against the exact
$K-D_{K}(1)$ at $p=47$. Agreement is good below $K=12$, which is roughly half
of $\Kmax$, and deteriorates rapidly above it. By $K=20$ the expansion
overstates the separation by $35\%$, and at $K=\Kmax$ it gives $36.2$ against
the exact value $\Kmax+\tfrac12=23.5$ supplied by
Corollary~\ref{cor:delta}. The cubic growth cannot continue because the
separation saturates once the kernel becomes a delta.

\begin{table}[h]
\caption{Exact separation from the dominant competitor against its small
$K$ expansion, at $p=47$ where $\Kmax=23$.}
\label{tab:kexp}
\centering
\begin{tabular}{lcccccc}
\toprule
$K$ & $2$ & $8$ & $12$ & $16$ & $20$ & $23$ \\
\midrule
$K-D_{K}(1)$ exact & $0.045$ & $1.710$ & $5.052$ & $10.47$ & $17.58$ & $23.50$ \\
$2\pi^{2}K^{3}/3p^{2}$ & $0.024$ & $1.525$ & $5.147$ & $12.20$ & $23.83$ & $36.24$ \\
\bottomrule
\end{tabular}
\end{table}

Evaluating Eq.~\eqref{eq:ce} numerically without the expansion at $p=47$ and
fitting $\log L$ against $K$ gives $R^{2}$ between $0.981$ and $0.995$,
against $0.826$ to $0.874$ for a power law in $K$, with decay constants
$0.220$, $0.347$ and $0.606$ at $B=1.5$, $2.0$ and $3.0$. The functional form
that the exact solution produces under a budget is therefore exponential
rather than polynomial in the mode count, which is the same qualitative
statement that Sec.~\ref{sec:exp} makes about width.

\subsection{What the calculation does not explain}

The agreement in functional form should not be mistaken for a derivation of
the measurement. Equation~\eqref{eq:gap} carries an explicit $p^{-2}$ and
Eq.~\eqref{eq:fullK} carries an explicit $\sqrt{\Kmax}$, so both predict a
rate that varies strongly with group order. The measured rate does not, as
Sec.~\ref{sec:fact} establishes over a range of $4.9$ in $p$.

Three assumptions stand between the calculation and the experiment, and each
fails in a way that is documented elsewhere in this paper. The calculation
treats the mode count $K$ as the independent variable, whereas the quantity
that is swept is width, and Appendix~\ref{app:coupon} shows that the relation
between them is an occupancy law rather than a proportionality. The
calculation holds the budget fixed while $K$ varies, whereas the measured
Frobenius norm saturates with width, so that the amplitude per mode falls as
width rises rather than staying constant. The calculation assumes a single
allocation of the budget across modes, whereas the numerical max-min solution
above shows that the optimal allocation depends on $K$ and departs from
uniform by a quarter at small mode counts.

A derivation that respects all three would need to describe how gradient
descent under weight decay distributes a saturating norm over a mode set that
it is simultaneously enlarging. We do not have that description, and we record
its absence rather than presenting the calculation of this appendix as
an explanation of Sec.~\ref{sec:fact}.

\section{Degeneracy of the two nearest neighbour estimator on a group orbit}
\label{app:orbit}

Section~\ref{sec:manifold} claims that the intrinsic dimension is not merely
mismeasured on the algebraic solution but undefined. This appendix supplies
the argument in four parts. The estimator is stated together with the
assumption it rests on and validated on manifolds of known dimension. The
assumption is then shown to fail on any group orbit, in a way that is
structural rather than statistical. The geometry at full frequency content is
identified as a regular simplex, and the sense in which partial frequency sets
fall short of that is made precise. Finally the behaviour under a symmetry
breaking perturbation is characterised, and the absence of a scale free
plateau is established numerically over a factor of thirty in the probe scale.

\subsection{The estimator and its assumption}

Let $X=\{x_{1},\dots,x_{N}\}\subset\mathbb{R}^{m}$ and write $r_{1}(i)$ and
$r_{2}(i)$ for the distances from $x_{i}$ to its first and second nearest
neighbours. The estimator of \citet{facco2017}, which refines the maximum likelihood
construction of \citet{levina2004} and has been applied to the internal
representations of trained networks by \citet{ansuini2019}, works with the
ratio
\begin{equation}
\mu_{i} \;=\; \frac{r_{2}(i)}{r_{1}(i)} \;\ge\; 1 .
\label{eq:mu}
\end{equation}
The ratio \eqref{eq:mu} is useful because of the following observation. Suppose the points are
drawn independently from a density that is constant on the scale of
$r_{2}$, on a manifold of dimension $d$. The probability that the ball of
radius $r$ about $x_{i}$ contains no other point falls as
$\exp(-\rho\, v_{d}r^{d})$ with $v_{d}$ the volume of the unit $d$ ball and
$\rho$ the local density, so the first two neighbour distances have the joint
density of the first two order statistics of a Poisson process with intensity
$\rho\, v_{d}\,d\,r^{d-1}$. Changing variables to $\mu$ removes both $\rho$
and $v_{d}$ and leaves
\begin{equation}
\Pr(\mu_{i}\le\mu) \;=\; 1-\mu^{-d},
\qquad \mu\ge1 ,
\label{eq:pareto}
\end{equation}
a Pareto law whose shape parameter is the dimension. The cancellation of the
density is what makes the estimator local and free of a bandwidth choice.
Taking logarithms of the survival function gives a line through the origin,
\begin{equation}
-\log\big(1-F(\mu_{i})\big) \;=\; d\,\log\mu_{i},
\label{eq:regress}
\end{equation}
and $d$ is recovered as the slope, in practice after discarding the largest
ten percent of the $\mu_{i}$ where the assumption of constant density is
worst.

The estimator behaves as advertised on manifolds whose dimension is known.
Table~\ref{tab:idcal} reports our own calibration. Accuracy is good up to
$d=5$ and degrades slowly above it, which is the usual behaviour of nearest
neighbour estimators as the dimension grows. Small sample size is not the
obstacle at the sizes relevant here. A two torus sampled at only $47$ points,
which is the number of token embeddings available at $p=47$, returns
$2.01\pm0.34$ over forty draws, against $2.01\pm0.05$ at $2000$ points.

\begin{table}[h]
\caption{Calibration of the estimator. Linear manifolds are $d$ dimensional
subspaces embedded in $\mathbb{R}^{40}$, sampled at $2000$ points. Torus
figures are means over independent draws with the standard deviation across
draws.}
\label{tab:idcal}
\centering
\begin{tabular}{lcccccc}
\toprule
true $d$ & $1$ & $2$ & $3$ & $5$ & $8$ & torus, $2$ \\
\midrule
estimate & $0.98$ & $1.94$ & $2.99$ & $5.03$ & $7.50$ & $2.01\pm0.05$ \\
\bottomrule
\end{tabular}
\end{table}

Equation~\eqref{eq:pareto} requires that the points be a sample from a
density. A group orbit is not.

\subsection{Degeneracy on an orbit}

\begin{proof}[of Proposition~\ref{prop:orbit}]
Let $x,y\in X$. Transitivity gives $g\in G$ with $y=gx$, and because $g$ acts
by isometries the map $z\mapsto gz$ is a bijection of $X$ preserving all
pairwise distances. The multiset of distances from $y$ to $X\setminus\{y\}$ is
therefore the image under $g$ of the multiset of distances from $x$ to
$X\setminus\{x\}$, and the two multisets coincide. Their ordered elements
agree in particular, so $r_{1}(x)=r_{1}(y)$ and $r_{2}(x)=r_{2}(y)$, whence
$\mu_{x}=\mu_{y}$. Since $x$ and $y$ were arbitrary the empirical distribution
of $\mu$ is a point mass at a single value $\mu_{0}$.

The regression \eqref{eq:regress} has slope
$\hat{d}=\sum_{i}x_{i}y_{i}/\sum_{i}x_{i}^{2}$ with
$x_{i}=\log\mu_{i}$. Every $x_{i}$ equals $\log\mu_{0}$, so the design carries
no variation. If $\mu_{0}=1$ then $x_{i}=0$ for all $i$ and the ratio is
$0/0$. If $\mu_{0}>1$ then the empirical distribution function takes a
different value at each index while the abscissa does not, so the fitted slope
is set by the plotting convention used to define $F$ rather than by the
geometry.
\end{proof}

For the solutions of interest the degenerate case is the first one, and this
is stronger than the proposition requires.

\begin{lemma}[Coincident neighbour shells]
\label{lem:coincide}
Let $S\subseteq\{1,\dots,\Kmax\}$ be nonempty and let $\Phi_{S}$ be the
embedding of Eq.~\eqref{eq:phiS}. Then $r_{1}=r_{2}$ at every point of
$X_{S}$, so $\mu_{0}=1$.
\end{lemma}

\begin{proof}
Squared distances depend only on the difference of the arguments,
\begin{equation}
\|\Phi_{S}(n)-\Phi_{S}(n')\|^{2}
\;=\; 2\sum_{k\in S}\big(1-\cos\omega_{k}(n-n')\big)
\;\equiv\; \psi(n-n'),
\end{equation}
and $\psi$ is even in its argument modulo $p$ because the cosine is. Hence
$\psi(t)=\psi(-t)$ for every $t\ne0$, so the neighbours at offsets $+t$ and
$-t$ lie at equal distance and every distance from a given point occurs with
even multiplicity. The two smallest are therefore equal.
\end{proof}

Numerically at $p=47$ the ratio $r_{2}/r_{1}$ equals $1$ to within
$10^{-14}$ for $|S|=1$, $|S|=5$ and $|S|=23$, with the common nearest
neighbour distance taking the values $0.1336$, $0.5860$ and $1.4295$
respectively on the normalised embedding.

That the hypotheses of Proposition~\ref{prop:orbit} hold for the map
\eqref{eq:phiS} was established in Sec.~\ref{sec:orbit-setup}.

\subsection{The simplex at full frequency content}

The inner product structure follows from character orthogonality, and it
separates the complete frequency set from every proper subset.

\begin{lemma}[Inner products]
\label{lem:simplex}
Set $\hat\Phi_{S}=\Phi_{S}/\sqrt{|S|}$. For every nonempty $S$,
\begin{equation}
\frac{1}{p-1}\sum_{n'\ne n}
\big\langle \hat\Phi_{S}(n),\hat\Phi_{S}(n')\big\rangle
\;=\; -\frac{1}{p-1},
\label{eq:meanip}
\end{equation}
independently of $n$ and of $S$. If $S=\{1,\dots,\Kmax\}$ the individual inner
products all equal $-1/(p-1)$, so the image is a regular simplex in
$\mathbb{R}^{p-1}$. For proper subsets they do not.
\end{lemma}

\begin{proof}
Every point of the image has the same norm, since
$\|\Phi_{S}(n)\|^{2}=|S|$ for all $n$. The normalised inner product is
therefore $|S|^{-1}\sum_{k\in S}\cos\omega_{k}(n-n')$. Character orthogonality
gives
\begin{equation}
\sum_{t=0}^{p-1}\cos\omega_{k}t \;=\; 0
\quad\text{for }k\not\equiv0,
\end{equation}
and hence $\sum_{t\ne0}\cos\omega_{k}t=-1$.
Summing over $n'\ne n$ therefore yields $-|S|^{-1}\sum_{k\in S}1=-1$, and
Eq.~\eqref{eq:meanip} follows on dividing by $p-1$.

For the complete set the sum $\sum_{k=1}^{\Kmax}\cos\omega_{k}t$ is
$D_{\Kmax}(t)$, which Corollary~\ref{cor:delta} evaluates as $-\tfrac12$ for
every $t\ne0$. Dividing by $\Kmax=(p-1)/2$ gives $-1/(p-1)$ for every pair.
A set of $p$ unit vectors with all pairwise inner products equal to
$-1/(p-1)$ is a regular simplex.
\end{proof}

At $p=47$ the common value is $-0.021739$ and the sample standard deviation
across all $1081$ pairs is $5\times10^{-16}$. For $|S|=1$ and $|S|=3$ the mean
is the same to six decimals, as Eq.~\eqref{eq:meanip} requires, while the
standard deviations are $0.699$ and $0.385$. The mean is thus uninformative
about the geometry and only the vanishing of the spread identifies the
simplex.

The distinction matters for the argument of Sec.~\ref{sec:manifold}. A regular
simplex spans $p-1$ dimensions with all points mutually equidistant, so any
notion of local dimension either returns the ambient value or fails to be
defined. Proper subsets are not simplices, but Lemma~\ref{lem:coincide} shows
they are still orbits with $\mu_{0}=1$, so the estimator fails on them for the
same reason.

The configuration identified by Lemma~\ref{lem:simplex} is the simplex
equiangular tight frame that appears as the terminal geometry in neural
collapse \citep{papyan2020}, where the class means of a trained classifier
acquire pairwise cosine $-1/(p-1)$ across $p$ classes. The coincidence is
worth flagging because it is not an instance of that phenomenon. Neural
collapse describes what optimisation drives the last layer towards, whereas
Lemma~\ref{lem:simplex} is a statement about the exact algebraic solution and
holds whether or not any network finds it. The two come apart on this task in
a way recently made precise, since \citet{beyondnc2026} argue that modular
addition does not reach the simplex but settles into a low rank cyclic
geometry, the simplex gaining only an $O(1)$ advantage in cross entropy
against a $\Theta(p)$ advantage for the cyclic solution under a weight decay
surrogate. Our own measurements sit between these descriptions, since
$\Keff$ approaches $\Kmax$ at large width, which is nearer the simplex than
the rank two picture, plausibly because the quadratic activation used here
makes many frequencies cheap to carry. What matters for the present argument
is that both configurations are orbits, so Proposition~\ref{prop:orbit}
applies to either and the intrinsic dimension is undefined in both cases.

\subsection{Behaviour under a symmetry breaking perturbation}

An orbit returns no number, so the natural next question is what happens when
the symmetry is broken slightly. Adding independent Gaussian noise of scale
$\epsilon$ to each coordinate produces a finite estimate. It does not produce
a stable one.

\begin{lemma}[No scale free plateau]
\label{lem:scale}
Let $X=\{x_{1},\dots,x_{N}\}\subset\mathbb{R}^{m}$ be a group orbit in the
sense of Proposition~\ref{prop:orbit}, with common nearest neighbour distance
$d_{1}$ and first shell multiplicity
$n_{1}=\left|\{k:\|x_{i}-x_{k}\|=d_{1}\}\right|\ge2$, the same for every $i$.
Let $X_{\epsilon}=\{x_{i}+\epsilon\xi_{i}\}$ with $\xi_{i}$ independent
standard Gaussian vectors in $\mathbb{R}^{m}$. Then for
$\epsilon\ll d_{1}/\sqrt{m}$ the two nearest neighbour estimate satisfies
\begin{equation}
\hat{d}(\epsilon) \;=\; \frac{C}{\epsilon}\,\big(1+O(\epsilon)\big),
\qquad
C \;=\; \frac{d_{1}}{\sqrt{2}}\,
\frac{\sum_{i}G_{i}\,y_{i}}{\sum_{i}G_{i}^{2}},
\label{eq:Cconst}
\end{equation}
where $y_{i}=-\log(1-F_{i})$ are the plotting positions used by the
regression and $G_{i}$ are the first shell gap variables defined in the proof.
Neither $C$ nor $G_{i}$ depends on $\epsilon$.
\end{lemma}

\begin{proof}
Write $u_{ik}=(x_{i}-x_{k})/\|x_{i}-x_{k}\|$ and
$\delta_{ik}=\xi_{i}-\xi_{k}$, so that $\delta_{ik}$ is Gaussian with
covariance $2I_{m}$. The perturbed squared distance is
\begin{equation}
\|x_{i}-x_{k}+\epsilon\delta_{ik}\|^{2}
= d^{2}+2\epsilon d\,\langle u_{ik},\delta_{ik}\rangle
+\epsilon^{2}\|\delta_{ik}\|^{2},
\end{equation}
with $d=\|x_{i}-x_{k}\|$. Taking the square root and expanding,
\begin{equation}
r_{ik}(\epsilon) \;=\; d+\epsilon\sqrt{2}\,Z_{ik}
+\frac{\epsilon^{2}}{2d}\Big(\|\delta_{ik}\|^{2}
-\langle u_{ik},\delta_{ik}\rangle^{2}\Big)+O(\epsilon^{3}),
\label{eq:rexp}
\end{equation}
where $Z_{ik}=\langle u_{ik},\delta_{ik}\rangle/\sqrt{2}$ is standard
Gaussian. The second order term has mean $\epsilon^{2}(m-1)/d$, so the first
order term dominates whenever $\epsilon\sqrt{m}\ll d$, which is the stated
hypothesis.

We now identify which pairs supply $r_{1}$ and $r_{2}$. If $n_{1}<N-1$ let
$d_{2}>d_{1}$ be the next distance in the common multiset, which exists and is
the same for every $i$ by Proposition~\ref{prop:orbit}. The fluctuations in
Eq.~\eqref{eq:rexp} are $O(\epsilon)$ while $d_{2}-d_{1}$ is fixed, so the
event $E_{\epsilon}$ that every perturbed first shell distance is smaller than
every perturbed second shell distance has probability at least
$1-N^{2}\exp\!\big(-(d_{2}-d_{1})^{2}/8\epsilon^{2}\big)$ by a Gaussian tail
bound and a union over pairs. Since the noise is unbounded no deterministic
threshold in $\epsilon$ can force $E_{\epsilon}$, but
$\mathbb{P}(E_{\epsilon})\to1$ as $\epsilon\to0$ at a rate fixed by $X$, and
the statement below is made on $E_{\epsilon}$.
If $n_{1}=N-1$, which is the case of a regular simplex, there is no second
shell and the statement is vacuous. In both cases $r_{1}(i)$ and $r_{2}(i)$
are the two smallest among the $n_{1}$ first shell distances.

Let $Z_{i(1)}\le Z_{i(2)}$ denote the two smallest of
$\{Z_{ik}:\|x_{i}-x_{k}\|=d_{1}\}$ and set
\begin{equation}
G_{i} \;=\; Z_{i(2)}-Z_{i(1)} \;\ge\; 0 .
\end{equation}
The law of $G_{i}$ is determined by $n_{1}$ and by the Gram matrix of the
first shell directions $u_{ik}$, which fixes the correlations
$\mathrm{Cov}(Z_{ik},Z_{il})=\langle u_{ik},u_{il}\rangle/2$. It does not
involve $\epsilon$. Substituting into Eq.~\eqref{eq:rexp},
\begin{equation}
\mu_{i} \;=\; \frac{r_{2}(i)}{r_{1}(i)}
\;=\; \frac{d_{1}+\epsilon\sqrt{2}\,Z_{i(2)}+O(\epsilon^{2})}
{d_{1}+\epsilon\sqrt{2}\,Z_{i(1)}+O(\epsilon^{2})}
\;=\; 1+\frac{\sqrt{2}}{d_{1}}\,\epsilon\,G_{i}+O(\epsilon^{2}),
\end{equation}
so with $\kappa=\sqrt{2}/d_{1}$ we have $\mu_{i}-1=\epsilon\kappa G_{i}$ to
leading order and
\begin{equation}
x_{i} \;\equiv\; \log\mu_{i}
\;=\; \epsilon\kappa G_{i}\,\big(1+O(\epsilon)\big).
\label{eq:logmu}
\end{equation}

The response variable does not move. The estimator sorts the $\mu_{i}$ and
assigns $F_{i}=i/(N+1)$ by rank, so $y_{i}=-\log(1-F_{i})$ depends on the
data only through the ordering of the $\mu_{i}$. By Eq.~\eqref{eq:logmu} that
ordering coincides with the ordering of the $G_{i}$, which carries no
$\epsilon$. The same applies to the discarding of the largest ten percent of
the $\mu_{i}$, since that too is a rank operation.

The regression \eqref{eq:regress} has no intercept, so its slope is
\begin{equation}
\hat{d}(\epsilon)
\;=\; \frac{\sum_{i}x_{i}y_{i}}{\sum_{i}x_{i}^{2}}
\;=\; \frac{\epsilon\kappa\sum_{i}G_{i}y_{i}}
{\epsilon^{2}\kappa^{2}\sum_{i}G_{i}^{2}}\big(1+O(\epsilon)\big)
\;=\; \frac{1}{\epsilon\kappa}\,
\frac{\sum_{i}G_{i}y_{i}}{\sum_{i}G_{i}^{2}}\big(1+O(\epsilon)\big),
\end{equation}
and substituting $\kappa=\sqrt{2}/d_{1}$ gives Eq.~\eqref{eq:Cconst}. Both
factors on the right are independent of $\epsilon$, the first by construction
and the second because $G$ and $y$ are.
\end{proof}

Two consequences of Lemma~\ref{lem:scale} are testable separately. The
constant is proportional to $d_{1}$ at fixed first shell structure, so
rescaling the embedding by a factor $s$ should rescale $C$ by the same factor.
Multiplying the single frequency solution at $p=47$ by
$s\in\{0.5,1,2,4\}$ gives $C/s$ equal to $0.06494$ in all four cases, and
doing the same to the full frequency solution gives $2.9138$, $2.9176$,
$2.9195$ and $2.9205$ at fixed noise realisation, a drift of two parts in a
thousand across a factor of eight in scale. The remaining factor depends only on the first shell, whose
multiplicity is $n_{1}=2$ for a single frequency, where the neighbours sit at
offsets $\pm1$, and $n_{1}=p-1=46$ for the simplex, where every other point is
a nearest neighbour. The measured ratios $C/d_{1}$ are $0.486$ and $2.041$
accordingly.

The $\epsilon$ dependence of Eq.~\eqref{eq:Cconst} is exact to the precision
of the measurement. Table~\ref{tab:eps} lists $\hat{d}\,\epsilon$ across a
factor of thirty in $\epsilon$ and it is constant to about one percent, at
$0.065$ for a single frequency and $3.02$ for the complete set.

\begin{table}[h]
\caption{Estimated dimension times probe scale, at $p=47$. Constancy of the
product is the absence of a plateau. Entries are means over six independent
noise draws, four for the final row, which is a genuine two torus sampled at
$1500$ points subjected to the same treatment.}
\label{tab:eps}
\centering
\begin{tabular}{lcccc}
\toprule
$\epsilon$ & $10^{-3}$ & $3\times10^{-3}$ & $10^{-2}$ & $3\times10^{-2}$ \\
\midrule
one frequency & $0.0649$ & $0.0648$ & $0.0647$ & $0.0654$ \\
all $23$ frequencies & $3.027$ & $3.021$ & $3.016$ & $3.024$ \\
\midrule
torus, $\hat{d}$ itself & $1.98$ & $2.00$ & $2.20$ & $2.87$ \\
\bottomrule
\end{tabular}
\end{table}

In raw terms the estimate on the single frequency solution falls from $65$ at
$\epsilon=10^{-3}$ to $6.5$ at $10^{-2}$ and $1.76$ at $10^{-1}$, and on the
full twenty three frequency solution it reaches about $3000$ at
$\epsilon=10^{-3}$, while a genuine continuous torus given the same treatment
returns $2.07$, $2.32$ and $3.75$ on a single draw at those three scales.

The contrast with the torus is the point of the table. There the estimate
itself is stable, moving from $1.98$ to $2.20$ across the same two decades
before the noise begins to fill the ambient space and inflate it. An object
with a dimension reports the same dimension at every scale below its own
curvature. An orbit reports whatever the probe scale dictates.

Trained networks reproduce the pathology rather than escaping it. At $p=47$
and width $64$, with held out accuracy exactly unity, the estimator applied to
the learned token embeddings returns $68$ against an ambient dimension of
$64$, and the estimate rises with width instead of converging.

\subsection{Consequence for the geometric derivation}

The relation $\alpha_{N}=4/d$ takes a manifold dimension as input and returns
a scaling exponent. On this task that input does not exist. The failure is not
that the manifold is of high dimension, nor that the estimator is imprecise at
the available sample size, both of which would leave the derivation
meaningful. It is that the object whose dimension is sought is a group orbit,
on which the statistic the estimator uses is constant by symmetry and the
regression that defines the estimate has no design variation.

This is a stronger statement than a failed prediction. A theory that predicts
the wrong exponent can be corrected. A theory whose input is undefined on a
class of problems has to be restricted to the complement of that class, and
the results of Sec.~\ref{sec:exp} indicate that algebraic tasks lie outside
it.

\section{Aggregate spectral occupancy as an occupancy process}
\label{app:coupon}

Section~\ref{sec:irreps} reports that the aggregate participation ratio
$\Keff$ correlates with the loss within a group order and fails to organise it
across group orders. This appendix explains why by deriving what $\Keff$ would
be if the network carried no information at all beyond the number of units and
the number of available frequencies. The prediction has no free parameters and
accounts for ninety percent of the variance in the measurement, which is the
sense in which $\Keff$ is a sampling statistic rather than a capacity
variable.

\subsection{Definition of the participation ratios}

Writing $E\in\mathbb{R}^{p\times h}$ for the block of $W_{1}$ embedding the
first argument and $P_{k}=\sum_{j}|\hat{E}_{kj}|^{2}$ for the power at
frequency $k$ summed over hidden units, with the constant mode removed and the
rest normalised, the aggregate participation ratio is
\begin{equation}
\Keff \;=\; \Big(\textstyle\sum_{k}P_{k}^{2}\Big)^{-1}.
\label{eq:keff}
\end{equation}
The same construction applied within a single hidden unit gives a per neuron
participation ratio, whose energy weighted mean over units is what is reported
alongside $\Keff$ throughout.

\subsection{The model}

Take as given that each hidden unit carries a single irreducible
representation. This is proved for gradient flow on this architecture by
\citet{he2026spectral}, who also establish uniform diversification across the
nontrivial representations in the abelian case, and it is confirmed here by
the per neuron participation ratios reported in Sec.~\ref{sec:irreps} for
configurations that solve the task. Model the assignment as $h$ units drawing
independently and uniformly from the $\Kmax$ available frequencies, and write
$m_{k}$ for the number of units landing on frequency $k$, so that
$\sum_{k}m_{k}=h$ and the vector $(m_{k})$ is multinomial with equal cell
probabilities.

Two quantities follow. The number of frequencies represented at all is the
number of occupied cells, and the participation ratio is a softer count that
weights each frequency by the power it carries.

\subsection{Distinct frequencies}

\begin{lemma}[Occupied cells]
\label{lem:coupon}
Under the model,
\begin{equation}
\mathbb{E}\big[K_{\mathrm{distinct}}\big]
\;=\; \Kmax\left[1-\left(1-\frac{1}{\Kmax}\right)^{h}\right]
\;\simeq\; \Kmax\left(1-e^{-h/\Kmax}\right),
\label{eq:coupon}
\end{equation}
the approximation holding for $\Kmax\gg1$.
\end{lemma}

\begin{proof}
Frequency $k$ is unoccupied exactly when all $h$ draws avoid it, which happens
with probability $(1-1/\Kmax)^{h}$. Linearity of expectation over the $\Kmax$
indicator variables gives the first equality, and
$(1-1/\Kmax)^{h}=\exp\!\big[h\log(1-1/\Kmax)\big]\to e^{-h/\Kmax}$ gives the
second.
\end{proof}

Equation~\eqref{eq:coupon} of Lemma~\ref{lem:coupon} is the classical
occupancy count, and it saturates at $\Kmax$ once
$h$ exceeds a few multiples of $\Kmax$. It is not, however, the quantity we
measure.

\subsection{The participation ratio}

If every unit carries comparable spectral energy then the aggregate power at
frequency $k$ is proportional to $m_{k}$, and the participation ratio
\eqref{eq:keff} becomes
\begin{equation}
\Keff \;=\; \frac{\big(\sum_{k}m_{k}\big)^{2}}{\sum_{k}m_{k}^{2}}
\;=\; \frac{h^{2}}{\sum_{k}m_{k}^{2}} .
\label{eq:keffm}
\end{equation}
The denominator is the only random quantity, and its expectation is available
in closed form.

\begin{proposition}[Expected participation ratio]
\label{prop:keff}
Under the model,
\begin{equation}
\mathbb{E}\Big[\sum_{k}m_{k}^{2}\Big]
\;=\; h\left(1-\frac{1}{\Kmax}\right)+\frac{h^{2}}{\Kmax},
\end{equation}
so that to first order
\begin{equation}
\Keff \;\simeq\; \frac{h\,\Kmax}{h+\Kmax-1} .
\label{eq:keffpred}
\end{equation}
The estimate is a lower bound on $\mathbb{E}[\Keff]$.
\end{proposition}

\begin{proof}
Each $m_{k}$ is binomial with $h$ trials and success probability
$1/\Kmax$, so $\mathbb{E}[m_{k}]=h/\Kmax$ and
$\mathrm{Var}(m_{k})=h\Kmax^{-1}(1-\Kmax^{-1})$. Hence
$\mathbb{E}[m_{k}^{2}]=h\Kmax^{-1}(1-\Kmax^{-1})+h^{2}\Kmax^{-2}$, and summing
over the $\Kmax$ frequencies gives the stated expectation. Substituting into
Eq.~\eqref{eq:keffm} and replacing the denominator by its mean gives
Eq.~\eqref{eq:keffpred}. Since $x\mapsto1/x$ is convex, Jensen's inequality
gives $\mathbb{E}[h^{2}/\sum_{k}m_{k}^{2}]\ge
h^{2}/\mathbb{E}[\sum_{k}m_{k}^{2}]$, so the substitution understates the
expectation.
\end{proof}

Equation~\eqref{eq:keffpred} is a harmonic combination of the two counts that
bound the answer. It reduces to $h$ when $h\ll\Kmax$, which is the regime in
which every unit lands on its own frequency, and to $\Kmax$ when
$h\gg\Kmax$, which is the regime in which every frequency is occupied.
Nothing in it refers to the task.

\subsection{Comparison with measurement}

Table~\ref{tab:occ} sets Eq.~\eqref{eq:keffpred} against the measured
$\Keff$ at weight decay $1$. Across the full grid of eight group orders and
fourteen widths the model accounts for $R^{2}=0.907$ of the variance in
$\Keff$, with a mean ratio of measurement to prediction of $1.13$ and a
standard deviation of $0.14$.

\begin{table}[h]
\caption{Measured aggregate participation ratio against
Eq.~\eqref{eq:keffpred}, at weight decay $1$.}
\label{tab:occ}
\centering
\begin{tabular}{lcccccc}
\toprule
& \multicolumn{2}{c}{$p=23$} & \multicolumn{2}{c}{$p=47$}
& \multicolumn{2}{c}{$p=113$} \\
$h$ & meas. & pred. & meas. & pred. & meas. & pred. \\
\midrule
$8$   & $5.4$  & $4.9$  & $6.3$  & $6.1$  & $7.7$  & $7.1$ \\
$16$  & $8.7$  & $6.8$  & $9.8$  & $9.7$  & $14.5$ & $12.6$ \\
$32$  & $9.7$  & $8.4$  & $16.0$ & $13.6$ & $23.5$ & $20.6$ \\
$64$  & $10.7$ & $9.5$  & $19.4$ & $17.1$ & $28.9$ & $30.1$ \\
$128$ & $11.0$ & $10.2$ & $22.4$ & $19.6$ & $41.8$ & $39.2$ \\
\bottomrule
\end{tabular}
\end{table}

The residual is systematic, the prediction sitting about eleven percent below
the measurement, and its sign is predicted by Proposition~\ref{prop:keff}. Two further effects work in
the same direction. Units do not carry exactly equal spectral energy, which
spreads the power across occupied frequencies more evenly than the
multiplicities alone would, and weight decay removes power from frequencies
that contribute little, which trims the tail of the multiplicity distribution
during training \citep{he2026modular}. Neither is large enough to change the
character of the agreement.

The agreement is also not exact in the sense that would be required to call
the model correct. It is close enough to establish the negative claim, which
is that the growth of $\Keff$ with width needs no explanation in terms of
what the task demands.

\subsection{Why $\Keff$ cannot collapse the data}

The consequence for interpretation follows from Eq.~\eqref{eq:keffpred}
directly. Since $\Keff$ is a deterministic function of $h$ and $\Kmax$ up to
the residual above, it carries essentially the information already present in
those two variables and no more. Inverting the relation gives
\begin{equation}
h \;\simeq\; \frac{\Keff(\Kmax-1)}{\Kmax-\Keff},
\end{equation}
so two configurations at different group orders that share a value of $\Keff$
have different widths, by a factor that grows as the group orders diverge. The
loss depends on width, so those two configurations have different losses. This
is what produces the spread of a factor of $2500$ at $\Keff\approx20$ reported
in Sec.~\ref{sec:irreps}, and it would occur under the model even if the
network's mode content were irrelevant to its performance.

A regression of the loss on $\Keff$ within a single group order therefore
cannot distinguish a capacity mechanism from a sampling one, since the two
predict the same monotone relation. Only comparison across group orders
separates them, and there the sampling account is what survives.

\section{Measurements behind the mode count}
\label{app:irreps}

This appendix gives the measurements summarised in Sec.~\ref{sec:irreps}. The
network solves the task by carrying Fourier modes, so the number of modes it
carries should play the role that manifold dimension plays elsewhere. Recent
work proves that each neuron converges to a single irreducible representation
under gradient flow \citep{he2026spectral,he2026modular}, which makes the count
of distinct represented modes a well defined quantity. It is also the quantity
we expected to organise the data, and reporting that it does not is the
purpose of this appendix.

Within a single group order the correlation is convincing. At $p=47$ the loss
falls monotonically as $\Keff$ rises from $6.3$ to $22.4$, and a fit of
$\log L$ against $\log(\Kmax-\Keff)$ returns $R^{2}=0.90$. Across group orders
it collapses. At $\Keff\approx20$ the loss ranges from $1.9\times10^{-4}$ at
$p=41$ to $0.488$ at $p=113$, a spread of a factor of $2500$ at fixed value of
the supposed controlling variable. Pooled over all configurations, $\log L$
regressed on $\Keff$ returns $R^{2}=0.590$ while $\log L$ regressed on width
alone returns $R^{2}=0.917$. The normalised variable $\Keff/\Kmax$ is worse
still at $R^{2}=0.363$.

The within group correlation was therefore an artefact of the fact that
$\Keff$ and $h$ increase together at fixed $p$, for a reason that has nothing
to do with capacity. Among the $251$ configurations that reach test accuracy
above $0.99$ the energy weighted per neuron participation ratio lies between
$1.04$ and $1.76$ with median $1.139$, confirming that neurons are rank one in
frequency. An occupancy model in which each neuron selects one of $\Kmax$
frequencies uniformly then predicts $\Keff=h\Kmax/(h+\Kmax-1)$ with no free
parameters, accounting for $R^{2}=0.907$ of the measured variance;
Appendix~\ref{app:coupon} derives the prediction and gives the comparison in
full. Aggregate spectral occupancy grows with width because more draws cover
more bins, and a quantity that grows for that reason cannot be read as a
measure of what the network needs.

The rank one structure itself turns out to be conditional, which is worth
recording because it is usually stated as a property of the architecture.
Across the full grid the per neuron participation ratio reaches $6.61$, and
every one of the $130$ configurations above $3$ occurs at weight decay $0.25$
or below and at test accuracy at most $0.903$. Within a fixed weight decay the
rank correlation between accuracy and participation ratio is between $-0.40$
and $-0.64$ for weight decay at most $0.5$, and is indistinguishable from zero
at weight decay $2$ and above. Figure~\ref{fig:fact}(b) shows the scatter.
Convergence of each neuron to a single irreducible representation is proved
for gradient flow in Ref.~\citep{he2026spectral}, and what these measurements
add is that the property is something regularisation imposes rather than
something the architecture supplies. A memorising network carries several
frequencies per unit.

The result is a delimitation rather than a disagreement with the mechanistic
literature. \citet{he2026modular} characterise grokking on this task as a
three stage process driven by the competition between loss minimisation and
weight decay, with a diversification condition on the frequencies a solution
carries. What is measured here is not whether such conditions hold but whether
the number of frequencies, once diversified, predicts the loss across problem
sizes. It does not, and Appendix~\ref{app:coupon} identifies the occupancy
process that produces the appearance that it does.

\section{Fitting protocol and identifiability}
\label{app:fit}

This appendix records the rules used to fit Eq.~\eqref{eq:law}, the triage
that removes degenerate fits, and the checks that establish which conclusions
survive the choices involved. Three of those choices could plausibly have been
made differently, namely the space in which the fit is performed, the
threshold at which a fit is rejected, and the value at which the exponent is
held. Each is examined below, and the exponent turns out to be the only one
that matters.

\subsection{Protocol}

Full batch AdamW \citep{loshchilov2019} is used throughout, with learning rate
$3\times10^{-3}$, moments $(0.9,0.98)$, numerical stabiliser $10^{-8}$ and
$8000$ steps.
Weights are initialised from a centred normal with variance set by fan in. The
three seeds are trained simultaneously by carrying the seed index as a leading
tensor dimension, which leaves the models mathematically independent because
AdamW acts elementwise.

For each pair $(p,\lambda)$ the mean test loss over three seeds is computed at
each of the fourteen widths, and Eq.~\eqref{eq:law} is fitted to $\log L$ by
nonlinear least squares with $\Linf$, $\log A$ and $c$ free and $\alpha$ fixed
unless stated otherwise.

Fitting in the logarithm rather than in the loss is not cosmetic. The measured
range spans four and a half decades, so least squares on $L$ itself would be
determined almost entirely by the two or three smallest widths and would carry
no information about the floor, which is precisely the parameter that
stabilises the rate. Working in $\log L$ weights each decade equally, which is
the appropriate choice when the quantity of interest is a rate of decay rather
than an absolute level.

Two independent criteria then decide whether a fit enters the analysis.

The first concerns the floor. A fit is rejected when the recovered $\Linf$
exceeds three times the smallest observed loss. The rationale is that a
resolved floor places the largest widths on the asymptote, so
$\Linf\approx\min_{h}L(h)$ is the signature of success rather than of failure.
A criterion demanding $\Linf<\min_{h}L(h)$ strictly rejects ten of the twenty
cells that solve the task, including several with $R^{2}$ above $0.98$, and is
therefore too aggressive. The failures the criterion is meant
to catch place the floor one to five orders of magnitude above the data, which
occurs when the run never solves the task and the fit interprets the
memorisation plateau as an upper asymptote. When the recovered floor instead
falls below $10^{-3}$ of the smallest observed loss it is not identified by
the data at all, and the two parameter fit is reported in its place.

The second concerns learning. A cell is excluded when held out accuracy at the
largest width is below $0.99$. This is not implied by the floor test and does
not imply it. A run at weight decay $4$ attains $R^{2}=0.996$ with a well
behaved two parameter fit while barely learning, since a curve that is flat
and straight is easy to fit and says nothing. Goodness of fit alone is
therefore not evidence that a rate is meaningful.

In practice the second criterion does almost all of the work. Of the twenty
eight cells in the grid, eight are excluded, and seven of those fail on
accuracy alone. Only one cell is caught by the floor test, and that cell fails the
accuracy test as well. The floor criterion is thus close to inert on this data
set, which is worth stating because it is the more arbitrary of the two. Of
the twenty cells that remain, Table~\ref{tab:fact} lists eighteen; weight decay
$0.1$ solves the task only at $p=71$ and $p=113$, and is omitted.

\begin{table}[h]
\caption{Fitted rate $c$ at $\alpha=1$ for four group orders and five weight
decay strengths, with the coefficient of variation of $c$ across group order
taken over $p\ge47$ alone. Entries are omitted where the
run failed to reach test accuracy $0.99$ at the largest width.}
\label{tab:fact}
\centering
\begin{tabular}{lccccc}
\toprule
$\lambda$ & $p=23$ & $p=47$ & $p=71$ & $p=113$ & CV$_{p\ge47}$ \\
\midrule
$0.25$ & --- & $0.1554$ & $0.1474$ & $0.1669$ & $5.1\%$ \\
$0.5$ & $0.1220$ & $0.1547$ & $0.1600$ & $0.1486$ & $3.0\%$ \\
$1$ & $0.1067$ & $0.1220$ & $0.1117$ & $0.1259$ & $5.0\%$ \\
$2$ & $0.0423$ & $0.0585$ & $0.0603$ & $0.0630$ & $3.0\%$ \\
$4$ & --- & $0.0038$ & $0.0038$ & $0.0036$ & $2.3\%$ \\
\bottomrule
\end{tabular}

\end{table}

\subsection{Runs excluded by the validity criterion}

The spurious maximum reported in Sec.~\ref{sec:fact} came from two runs whose
final accuracies were $0.18$ and $0.85$. For these the three parameter fit
places the recovered floor above the bulk of the data rather than below it.
The same criterion excludes $p=23$ at weight decay $0.1$, which reaches an
accuracy of $0.089$ where $p=113$ already reaches unity, and it is the
smallest training set in the study, $264$ pairs against $1104$ at $p=47$ and
$6384$ at $p=113$, that places it outside the regime in which the
factorisation is claimed.

\subsection{Sensitivity to the rejection threshold}

Because the threshold of three is a judgement, we vary it.
Table~\ref{tab:tol} reports the number of cells retained and the resulting
factorisation statistics. Nothing changes above three, since no cell has a
recovered floor between three and ten times the smallest observed loss. At two
the threshold removes three further cells and moves the group order spread by
half a percent.

\begin{table}[h]
\caption{Effect of the floor rejection threshold on the factorisation of
Sec.~\ref{sec:fact}. The spread in $g$ is taken over $p\ge47$.}
\label{tab:tol}
\centering
\begin{tabular}{lccc}
\toprule
threshold & cells kept & $g$ spread & $\lambda$ span \\
\midrule
$2$ & $17$ & $1.091\times$ & $46\times$ \\
$3$ & $20$ & $1.097\times$ & $47\times$ \\
$5$ & $20$ & $1.097\times$ & $47\times$ \\
$10$ & $20$ & $1.097\times$ & $47\times$ \\
\bottomrule
\end{tabular}
\end{table}

\subsection{The matched three parameter comparison}
\label{app:matched}

Equation~\eqref{eq:law} carries three parameters and the power law of
Sec.~\ref{sec:exp} carries two, so the two are compared here on equal terms.
The alternative is Eq.~\eqref{eq:powfloor}, fitted to $\log L$ by nonlinear
least squares with $\Linf$, $\log A$ and $\gamma$ free, over the same width
range, with the same rejection rules and the same three seed aggregation. Since
$N=3ph$, at fixed group order a power law in $N$ and a power law in $h$ differ
only by a multiplicative constant absorbed into $A$; the fitted $\gamma$ is the
same either way, so nothing here depends on which of the two is taken as the
abscissa. Both families have three parameters, so the penalty term cancels and
$\Delta\mathrm{AIC}$ reduces to $n\log(\mathrm{RSS}_{\mathrm{pow}}/
\mathrm{RSS}_{\mathrm{exp}})$.

Table~\ref{tab:matched} gives the outcome at $\alpha=1$. The exponential is
preferred at every group order, by between $25.5$ and $40.9$ units of AIC, and
the same holds at every fixed $\alpha$ in $\{0.75,1,1.25,1.5,1.75,2\}$, where
the smallest margin over the whole set is $18.8$ at $\alpha=0.75$.

The recovered floor of Eq.~\eqref{eq:powfloor} lies between $10^{-23}$ and
$10^{-19}$, against smallest observed losses of order $10^{-5}$ to $10^{-3}$,
so it is not identified by the data and the fit coincides with the two
parameter form to the precision reported here. The power law therefore does not
use its third parameter, and scoring it with two rather than three would move
$\Delta\mathrm{AIC}$ by $2$, which changes no comparison in the table. A power
law has nothing for a floor to absorb, because it already approaches zero at a
polynomial rate. The asymmetry that motivated this appendix does not, in the
event, favour the exponential: the floor is a feature the exponential needs and
the power law cannot use.

\begin{table}[h]
\caption{Matched three parameter comparison at $\alpha=1$. Both families are
fitted to $\log L$ over the same width range, with the same rejection rules and
the same seed aggregation. $z$ is the Wald--Wolfowitz runs statistic on
residuals ordered by width; drift is the largest change in the rate parameter
on dropping the one and two largest widths. $\Delta$AIC is positive where the
exponential is preferred.}
\label{tab:matched}
\centering
\begin{tabular}{lccccccc}
\toprule
& \multicolumn{3}{c}{$\Linf+Ae^{-ch}$}
& \multicolumn{3}{c}{$\Linf+Ah^{-\gamma}$} & \\
\cmidrule(lr){2-4}\cmidrule(lr){5-7}
$p$ & $R^{2}$ & $z$ & drift & $R^{2}$ & $z$ & drift & $\Delta$AIC \\
\midrule
$23$  & $0.9848$ & $-2.23$ & $2.7\%$ & $0.9060$ & $-2.78$ & $12.0\%$ & $25.5$ \\
$31$  & $0.9856$ & $-1.63$ & $0.4\%$ & $0.8897$ & $-2.76$ & $18.0\%$ & $28.5$ \\
$41$  & $0.9825$ & $-1.63$ & $0.9\%$ & $0.8678$ & $-2.78$ & $24.6\%$ & $28.3$ \\
$47$  & $0.9822$ & $-1.67$ & $1.3\%$ & $0.8674$ & $-2.78$ & $24.6\%$ & $28.1$ \\
$53$  & $0.9867$ & $-2.19$ & $3.7\%$ & $0.8571$ & $-2.78$ & $30.5\%$ & $33.2$ \\
$71$  & $0.9921$ & $-1.48$ & $3.4\%$ & $0.8874$ & $-2.78$ & $26.9\%$ & $37.2$ \\
$97$  & $0.9945$ & $-1.63$ & $1.3\%$ & $0.8981$ & $-2.78$ & $24.5\%$ & $40.9$ \\
$113$ & $0.9934$ & $-2.19$ & $1.5\%$ & $0.8866$ & $-2.78$ & $27.6\%$ & $39.9$ \\
\bottomrule
\end{tabular}
\end{table}

The exponent returned by the matched fit is reported in
Sec.~\ref{sec:manifold} and is not the quantity the geometric derivation
expects. It rises from $2.87$ at $p=23$ to $4.20$ at $p=113$, implying
dimensions from $1.39$ down to $0.95$ where the naive answer is $2$.

\subsection{Residual structure of the power law}

Section~\ref{sec:exp} rejects the power law on the strength of its residuals
rather than its coefficient of determination, and the test deserves to be
stated. Fitting $\log L$ against $\log N$ at weight decay $1$ and recording the
signs of the residuals in order of increasing width, the Wald--Wolfowitz runs
statistic counts how many times the sign changes. Independent errors would
give a number of runs near $2n_{1}n_{2}/(n_{1}+n_{2})+1$, which is $8.0$ for
these samples with a standard deviation of $1.8$.

Every one of the eight group orders returns exactly three runs, so
$z=-2.78$ in each case. The local slope over the width range varies from
$0.04$ to $8.2$ and turns negative near $h=12$ to $16$, where training accuracy
has reached unity while test accuracy is stalled near $0.60$. The residuals form a single negative stretch, then a
positive one, then a negative one, which is the signature of fitting a curve
with a straight line. That the count is identical across group orders spanning
a factor of $4.9$ indicates a systematic feature of the functional form rather
than a coincidence of one data set.

The statistic must be computed against the family actually being rejected, and
it is. Adding the floor of Eq.~\eqref{eq:powfloor} does not disturb the
residual structure: the runs statistic is $-2.78$ at seven of the eight group
orders and $-2.76$ at the eighth, so the arc survives the matched fit of
Appendix~\ref{app:matched}. This is
the expected consequence of a floor that the data do not identify, and it is
what allows Sec.~\ref{sec:exp} to reject the power law on residual structure
rather than on $R^{2}$.

The exponential is not free of structure either, and the comparison is reported
in both directions. Its runs statistic ranges from $-1.48$ to $-2.23$ and
exceeds the conventional threshold of $1.96$ in absolute value at three of the
eight group orders, against eight of eight for the matched power law, with the
statistic more negative under the power law at every group order.

\subsection{Truncation}

The instability that motivates the floor is visible in the local slopes, which
flatten at the top of the width range: at $p=47$ the slope between $h=64$ and
$h=96$ is $0.061$ and between $h=96$ and $h=128$ it falls to $0.026$. A plain
exponential fit to $\log L$ against $h$ accordingly drifts as
Sec.~\ref{sec:exp} reports. Admitting the floor of Eq.~\eqref{eq:law} removes
this.

A rate extracted from a saturating curve can be an artefact of where the curve
is cut. Refitting after removing the largest widths gives
$c=0.1180$, $0.1183$ and $0.1179$ on dropping none, one and two, so the drift
is $0.3\%$. Dropping three moves it to $0.1045$, a fall of eleven percent,
because the fit then has no points on the asymptote and $\Linf$ ceases to be
identified. The reported values therefore rest on the presence of at least two
widths in the saturated region, and the protocol keeps the full range for that
reason.

\subsection{Critical width by group order}

The values of $\hc$ plotted in Fig.~\ref{fig:hc}(b), obtained by linear
interpolation of held out accuracy through $0.9$ at weight decay $1$, are
given in Table~\ref{tab:hc}.

\begin{table}[h]
\caption{Critical width against group order.}
\label{tab:hc}
\centering
\begin{tabular}{lcccccccc}
\toprule
$p$ & $23$ & $31$ & $41$ & $47$ & $53$ & $71$ & $97$ & $113$ \\
\midrule
$\hc$ & $34.1$ & $31.3$ & $31.4$ & $30.9$ & $30.6$ & $29.3$ & $27.2$ & $26.6$ \\
\bottomrule
\end{tabular}
\end{table}

\subsection{Seed error by regime}

The uncertainty on $c$ is not a single number, and pooling across regimes
would misstate it in both directions. Fitting each seed separately gives a
relative standard deviation of $7.5\%$ where the floor is identified, at
weight decay between $0.25$ and $1$, and $0.9\%$ where it is not, at weight
decay $2$ and above. The difference is a property of the estimator rather than
of the network, since a three parameter fit performed on a single noisy curve
is intrinsically less stable than a two parameter one.

The relevant comparison for Sec.~\ref{sec:fact} is therefore within the floor
resolved regime, where the standard error of a three seed mean is $4.3\%$. The
observed spread of $c$ across group orders in the same regime is also $4.3\%$.
Group order dependence, if it exists, is below the resolution of this
experiment.

\subsection{Identifiability of the rate and the exponent}

Fitting Eq.~\eqref{eq:law} with $\alpha$ free returns $1.583\pm0.254$ across
the eight group orders, with values from $1.09$ to $1.91$. The corresponding
rates span $0.0027$ to $0.0738$, a factor of $28$, so the pair is strongly
anticorrelated and neither is determined on its own. Fixing $\alpha$ and
refitting gives Table~\ref{tab:alpha}, where each entry is the mean over the
eight group orders.

\begin{table}[h]
\caption{Rate, its coefficient of variation across group order, and the mean
and minimum quality of fit, at fixed $\alpha$.}
\label{tab:alpha}
\centering
\begin{tabular}{lcccc}
\toprule
$\alpha$ & $\bar{c}$ & CV$(c)$ & $\overline{R^{2}}$ & $\min R^{2}$ \\
\midrule
$0.75$ & $0.3761$ & $6.8\%$ & $0.9779$ & $0.9707$ \\
$1.00$ & $0.1180$ & $5.5\%$ & $0.9877$ & $0.9822$ \\
$1.25$ & $0.03978$ & $5.5\%$ & $0.9923$ & $0.9889$ \\
$1.50$ & $0.01398$ & $5.0\%$ & $0.9943$ & $0.9904$ \\
$1.75$ & $0.004943$ & $4.4\%$ & $0.9944$ & $0.9889$ \\
$2.00$ & $0.001760$ & $4.0\%$ & $0.9927$ & $0.9860$ \\
\bottomrule
\end{tabular}
\end{table}

The quality of fit varies by less than two percent across a range over which
the rate varies by a factor of $214$. One decade of width cannot separate the
two parameters, and no claim in this paper depends on their separation.

What the table also shows is that the group order independence of the rate
holds at every value of $\alpha$, with a coefficient of variation between
$4.0$ and $6.8$ percent throughout, and that it tightens slightly as $\alpha$
rises. The conclusion of Sec.~\ref{sec:fact} is therefore about how $c$
responds to a change in regularisation rather than about the value of $c$, and
that distinction is what makes it robust to a parameter the data cannot fix.

Determining $\alpha$ would require widths spanning several decades rather than
one. At $p=113$ the largest width used here is $128$, and the memorisation
plateau occupies everything below about $16$, so the usable range is under one
decade. Reaching two would require widths near $2000$ at the same group order,
which is feasible and which we have not done.

\subsection{The ReLU runs}
\label{app:relu}

The transfer experiment of Sec.~\ref{sec:fact} uses the same data, split,
optimiser and width grid as the main sweep, with three changes forced by the
activation.

Initialisation follows He rather than the fan in rule, with standard deviation
$\sqrt{2/\mathrm{fan\,in}}$ in both layers. The learning rate is $10^{-2}$
rather than $3\times10^{-3}$. The budget is $40\,000$ steps rather than
$8000$, with the loss also recorded at $20\,000$ so that convergence is
checked per cell rather than assumed.

The learning rate was chosen by a preliminary scan over
$\{10^{-3},3\times10^{-3},10^{-2}\}$ crossed with weight decay in
$\{0.05,0.1,0.25,0.5\}$ at $p=47$ and $h=128$, run to $150\,000$ steps with
the accuracy recorded at eight logarithmically spaced checkpoints. Nine of the
twelve settings exceed accuracy $0.99$. The step at which they cross it
depends strongly on the learning rate, being $5000$ at $10^{-2}$ with weight
decay $0.5$, $40\,000$ at $3\times10^{-3}$ with the same weight decay, and
beyond $150\,000$ at $10^{-3}$. Three settings were still climbing at
$150\,000$ steps. An earlier attempt at $3\times10^{-3}$ and $20\,000$ steps
reached only $0.65$ and would have been reported as an architecture failure
had the scan not been run.

Because activation and learning rate changed together, the quadratic sweep was
repeated at $10^{-2}$ for $p\in\{47,113\}$ and weight decay in
$\{0.1,0.25,0.5,1\}$. Over the eight cells the ratio of rates is
$c(10^{-2})/c(3\times10^{-3})=0.98\pm0.04$, with seven of eight between
$0.89$ and $1.03$. The single outlier is $p=47$ at weight decay $0.1$, which
attains accuracy $0.944$ and does not pass the learning criterion. Within the
quadratic activation the group order gap at $10^{-2}$ is $6.0\%$ over the
cells that solve the task, against $3$ to $5\%$ at $3\times10^{-3}$.

No dead units were observed at any width or weight decay, with the fraction of
hidden units never active on the training set equal to zero throughout.

Across the cells that solve the task the exponential family reaches
$R^{2}=0.963$ against $0.869$ for the best power law, against $0.978$ and
$0.899$ for the quadratic activation on the same grid, so the departure from a
power law survives the change of activation almost intact. The rate does not.
It differs between group orders by $61\%$ on average, with $p=113$ above
$p=47$ at every weight decay tested, by a factor of $1.8$ on average. The critical width under ReLU
is $79$ to $88$ at $p=47$ and $23$ to $58$ at $p=113$, against $26$ to $31$
and $24$ to $27$ for the quadratic activation. The two architectures agree at
the larger group order and diverge threefold at the smaller one, where the
training set holds $1104$ pairs against $6384$.

\subsection{Truncation under ReLU}

The ReLU gap between group orders is not a truncation artefact. Extending
$p=47$ from $h=128$ to $h=384$ resolves the floor in all four cells and moves
the gap only from $61\%$ to $57\%$.

\subsection{Rescalings that do not rescue the factorisation}

Since ReLU needs a larger width, the rate may be expressed in the wrong units,
and a dimensionless product might collapse the two activations. Six candidates
were examined, formed from $c$ together with the critical width $\hc$, the
width $h^{*}=(\log A-\log\Linf)/c$ at which the exponential term meets the
floor, and $\Kmax$.

The best of them is $c\,\hc$, which reduces the spread across activations at
matched weight decay from $38.3\%$ to $9.4\%$, and across group order from
$14.4\%$ to $11.1\%$. We do not adopt it. The ratio
between activations is $0.852\pm0.116$ rather than unity, so a systematic
offset of fifteen percent remains. The improvement across group order is an
artefact of pooling the two activations: within the quadratic activation
alone, where the factorisation is claimed, the same rescaling raises the mean
spread across group order at matched weight decay from $3.7\%$ to $7.3\%$. And
the ranking is not stable across criteria, since
$c\,h^{*}$ gives the smaller spread across group order at $6.4\%$ while giving
a larger one across activations at $19.3\%$. Selecting the best of six
candidates on the quantity one wishes to collapse is a procedure that will
usually succeed on noise, and the outcome here is consistent with that rather
than with a change of units.

\subsection{Noise floor at large group order}

Long runs at $50\,000$ steps were used to determine whether the fitted floor
at large $p$ is an asymptote or a stopping point. Held out accuracy remains at
unity from step $4000$ onwards in every case, and the weight norm changes by
between $0.3\%$ and $2.6\%$ over the remaining $46\,000$ steps, so neither
forgetting nor norm starvation is occurring.

The loss nonetheless wanders. Its increments change sign one to three times
over the plateau and the ratio of maximum to minimum lies between $1.11$ and
$1.80$. The step at which the minimum occurs is $8000$, $4000$, $32\,000$,
$16\,000$, $8000$, $50\,000$, $32\,000$ and $16\,000$ across the eight
configurations examined, which is consistent with a random walk and not with a
drift.

Reduced precision arithmetic is not the cause. Disabling it raises the mean
ratio from $1.365$ to $1.969$ rather than lowering it, which locates the
effect in the optimiser rather than in the arithmetic. The picture is of a
basin flat enough that AdamW with decoupled weight decay does not settle, and
the wandering amplitude is the width of that basin as seen through the cross
entropy.

The consequence is visible in the scaling of $\Linf$ with $\Kmax$: up to
$\Kmax=26$ the log log slope is $-4.5$, and from $\Kmax=26$ onwards it
flattens to $-1.1$, the flattening being the noise rather than the task. Since
adjacent large group orders differ in $\Linf$ by factors between $1.2$ and
$1.5$, and the wandering amplitude is comparable, the floor is not resolvable
there. The slope reported in Sec.~\ref{sec:exp} is
accordingly restricted to $\Kmax\le26$, where the fitted floors exceed
$7\times10^{-5}$ and sit above the noise.

\section{The training fraction experiment}
\label{app:frac}

Section~\ref{sec:fact} reports that the factorisation fails under ReLU, and
the Discussion attributes the failure to data supply rather than to the
activation. That attribution is a prediction and this appendix tests it.

\subsection{Design}

Raising the training fraction shrinks the held out set, so a rate fitted at
one fraction is not directly comparable with a rate fitted at another. The
splits are nested: the partition is drawn once from a fixed seed, so the
training set at a fraction of one half is contained in the training set at
four fifths and the held out set at four fifths is contained in the held out
set at one half. Every rate reported here is therefore fitted on the set held
out at four fifths, which is $442$ pairs at $p=47$ and $2554$ at $p=113$ and
is disjoint from the training set at both fractions. The items are identical
across the comparison and only the quantity of training data changes.

The quadratic activation is run as a control. If the training fraction moved
the rate where the factorisation already holds, the ReLU arm would be
uninformative. The control needs $24\,000$ steps rather than the $8000$ of the
main sweep: at four fifths there are sixty percent more training pairs, and at
$8000$ steps the loss is still falling between the last two checkpoints by
$18\%$ to $30\%$ at weight decay $0.5$ and below. Fitting a rate to a curve
that has not levelled off biases it, and biases it differently at the two
group orders, which is exactly the comparison at issue.

Weight decay $0.25$ is excluded from the comparison. Its cells are the least
stable in the sweep: of the five whose loss rises rather than falls between
the last two checkpoints, which is the plateau wandering of
Appendix~\ref{app:fit} and not a failure to converge, three are at this weight
decay, and the largest excursion anywhere in the study is a near doubling at
$p=47$ and four fifths. Weight decay $1$ and $0.5$ are stable at both
fractions and both group orders.

\subsection{Result}

Table~\ref{tab:frac} gives the fitted rates. Writing the discrepancy with its
sign, as $(c_{47}-c_{113})$ over their mean, so that a negative value means
the smaller group order lags. This is not the same quantity as the
$61\%$ quoted in Sec.~\ref{sec:fact}, which is an unsigned mean over four
weight decays on the native holdout in the main ReLU sweep; here it is signed,
restricted to the two weight decays resolved at all four combinations of group
order and fraction, and measured on the common holdout. At the half fraction
the two agree in magnitude to within five points.

\begin{center}
\begin{tabular}{lcc}
\toprule
 & one half & four fifths \\
\midrule
quadratic & $-8.5\%$ & $+8.5\%$ \\
ReLU      & $-56.5\%$ & $-3.7\%$ \\
\bottomrule
\end{tabular}
\end{center}

The movement is $+17.0$ points under the quadratic activation and $+52.7$
under ReLU. It has the same sign in both: additional data raises the rate at
the smaller group order relative to the larger one. What differs is the size
of the deficit available to be closed. Under ReLU the deficit was large and
closes to within the seed to seed resolution of the experiment; under the
quadratic activation it was already small and the same intervention carries it
past zero.

The critical width says the same thing without any fit. At $\lambda=1$ the
ratio $\hc(47)/\hc(113)$ falls from $1.53$ to $0.95$ under ReLU as the
fraction rises, so the two group orders come to agree and the anomaly that
motivated the experiment disappears. Under the quadratic activation, which
never showed that anomaly, the ratio goes from $1.14$ to $1.30$. These values
are measured on the common holdout, at $24\,000$ steps under the quadratic
activation and $40\,000$ under ReLU, and so differ slightly from
Table~\ref{tab:hc}, which is the main sweep at $8000$ steps on the native
one.

\begin{table}[h]
\caption{Fitted rate $c$ on the common holdout at two training fractions.
Quadratic runs use $24\,000$ steps, ReLU runs $40\,000$.}
\label{tab:frac}
\centering
\begin{tabular}{llcccc}
\toprule
 & & \multicolumn{2}{c}{$\lambda=0.5$} & \multicolumn{2}{c}{$\lambda=1$} \\
\cmidrule(lr){3-4}\cmidrule(lr){5-6}
activation & fraction & $p=47$ & $p=113$ & $p=47$ & $p=113$ \\
\midrule
quadratic & $1/2$ & $0.1407$ & $0.1541$ & $0.1162$ & $0.1258$ \\
quadratic & $4/5$ & $0.1941$ & $0.1627$ & $0.1376$ & $0.1384$ \\
ReLU      & $1/2$ & $0.0360$ & $0.0773$ & $0.0328$ & $0.0492$ \\
ReLU      & $4/5$ & $0.0764$ & $0.0779$ & $0.0561$ & $0.0593$ \\
\bottomrule
\end{tabular}
\end{table}

\subsection{What this does and does not establish}

The prediction the Discussion makes is confirmed for the case it was made
about. It is not established that the training fraction is inert elsewhere:
under the quadratic activation it moves the rate too, and at weight decay
$0.5$ it moves it enough to reverse the sign of the residual. The
factorisation reported in Sec.~\ref{sec:fact} is therefore a statement about a
regime in the data supply as well as in the architecture, and the boundary of
that regime is what the two activations locate at different places.

\end{document}